\documentclass[aos]{imsart}

\RequirePackage{amsthm,amsmath,amsfonts,amssymb}
\RequirePackage[authoryear]{natbib}
\RequirePackage[colorlinks,citecolor=blue,urlcolor=blue]{hyperref}
\RequirePackage{graphicx}
\RequirePackage{bm, mathtools}

\DeclareMathOperator*{\argmin}{arg\,min}
\newcommand{\norm}[1]{\lVert#1\rVert}
\usepackage{algorithm} 
\usepackage[noend]{algpseudocode}
\usepackage{todonotes}
\usepackage{bbold}
\startlocaldefs
\newtheorem{proposition}{Proposition}
\newtheorem{axiom}{Axiom}
\newtheorem{corollary}{Corollary}
\newtheorem{claim}[axiom]{Claim}
\newtheorem{theorem}{Theorem}
\newtheorem{lemma}[theorem]{Lemma}
\newtheorem{innercustomthm}{Theorem}
\newenvironment{customthm}[1]
  {\renewcommand\theinnercustomthm{#1}\innercustomthm}
  {\endinnercustomthm}

\theoremstyle{remark}

\endlocaldefs

\begin{document}

\begin{frontmatter}
\title{Multi-study Boosting: Theoretical Considerations for Merging vs. Ensembling}
\runtitle{Multi-study Boosting: Merging vs. Ensembling}

\begin{aug}
\author[A, B]{\fnms{Cathy} \snm{Shyr}\ead[label=e1]{}},
\author[C]{\fnms{Pragya} \snm{Sur}\ead[label=e3]{}},
\author[A, B]{\fnms{Giovanni} \snm{Parmigiani}\ead[label=e2]{}},
\and
\author[D]{\fnms{Prasad} \snm{Patil}\ead[label=e4]{}}
\address[A]{Department of Biostatistics,
Harvard T.H. Chan School of Public Health, Boston, MA, USA
\printead{e1}}

\address[B]{Department of Data Science, Dana-Farber Cancer Institute, Boston, MA, USA
\printead{e2}}

\address[C]{Department of Statistics, Harvard University, Cambridge, MA, USA 
\printead{e3}}

\address[D]{Department of Biostatistics, Boston University School of Public Health, Boston, MA, USA
\printead{e4}}

\end{aug}

\begin{abstract}
Cross-study replicability is a powerful model evaluation criterion that emphasizes generalizability of predictions. When training cross-study replicable prediction models, it is critical to decide between merging and treating the studies separately. We study boosting algorithms in the presence of potential heterogeneity in predictor-outcome relationships across studies and compare two multi-study learning strategies: 1) merging all the studies and training a single model, and 2) multi-study ensembling, which involves training a separate model on each study and ensembling the resulting predictions. In the regression setting, we provide theoretical guidelines based on an analytical transition point to determine whether it is more beneficial to merge or to ensemble for boosting with linear learners. In addition, we characterize a bias-variance decomposition of estimation error for boosting with component-wise linear learners. We verify the theoretical transition point result in simulation and illustrate how it can guide the decision on merging vs. ensembling in an application to breast cancer gene expression data.
\end{abstract}

\begin{keyword}
\kwd{replicability, multi-study, boosting}
\end{keyword}

\end{frontmatter}
\section{Introduction}
In settings where comparable studies are available, it is critical to simultaneously consider and systematically integrate information across multiple studies when training prediction models. Multi-study prediction is motivated by applications in biomedical research, where exponential advances in technology and facilitation of systematic data-sharing increased access to multiple studies (\cite{kannan2016public, manzoni2018genome}). When training and test studies come from different distributions, prediction models trained on a single study generally perform worse on out-of-study samples due to heterogeneity in study design, data collection methods, and sample characteristics. (\cite{castaldi2011empirical, bernau2014cross, trippa2015bayesian}). Training prediction models on multiple studies can address these challenges and improve the cross-study replicability of predictions. 

Recent work in multi-study prediction investigated two approaches for training cross-study replicable models: 1) merging all studies and training a single model, and 2) multi-study ensembling that involves training a separate model on each study and combining the resulting predictions. When studies are relatively homogeneous, \cite{patil2018training} showed that merging can lead to improved replicability over ensembling due to increase in sample size; as between-study heterogeneity increases, multi-study ensembling demonstrated preferable performance. While the trade-off between these approaches has been explored in detail for random forest (\cite{ramchandran2020tree}) and linear regression (\cite{guan2019merging}), none have examined this for boosting, one of the most successful and popular supervised learning algorithms.

 Boosting combines a powerful machine learning approach with classical statistical modeling. Its flexible choice of base learners and loss functions makes it highly customizable to many data-driven tasks including binary classification (\cite{freund1997decision}), regression (\cite{friedman2001greedy}) and survival analysis (\cite{wang2010buckley}). To the best of our knowledge, this work is the first to study boosting algorithms in a setting with multiple and potentially heterogeneous training and test studies. Existing findings on boosting are largely rooted in theories based on a single training study, and extensions of the algorithm to a multi-study setting often assume a subset of the training study shares the same distribution as the test study. \cite{buehlmann2006boosting} and \cite{tutz2007boosting} studied boosting with linear base learners and characterized an exponential bias-variance trade-off under the assumption that the training and test studies have the same predictor distribution. \cite{habrard2013boosting} proposed a boosting algorithm for domain adaptation with a single training study. \cite{dai2007boosting} proposed a transfer learning framework for boosting that uses a small amount of labeled data from the test study in addition to the training data to make classifications on the test study. This approach was extended to handle data from multiple training studies (\cite{yao2010boosting, bellot2019boosting}) and modified for regression (\cite{pardoe2010boosting}) and survival analysis (\cite{bellot2019boosting}). 

In this paper, we study boosting algorithms in a regression setting and compare cross-study replicability of merging versus multi-study ensembling. We assume a flexible mixed effects model with potential heterogeneity in predictor-outcome relationships across studies and provide theoretical guidelines to determine whether merging is more beneficial than ensembling for a given collection of training datasets. In particular, we characterize an analytical transition point beyond which ensembling exhibits lower mean squared prediction error than merging for boosting with linear learners. Conditional on the selection path, we characterize a bias-variance decomposition for the estimation error of boosting with component-wise linear learners. We verify the theoretical transition point results via simulations, and illustrate how it may guide practitioners' choice regarding merging vs. ensembling in a breast cancer application.

\section{Methods}
\subsection{Multi-study Setup}
We consider $K$ training studies and $V$ test studies that measure the same outcome and the same $p$ predictors. Each study has size $n_k$ with a combined size of $N = \sum_{k=1}^K n_k$ for the training studies and $N^{\text{Test}}=\sum_{k=K+1}^{K+V} n_k$ for the test studies. Let $Y_k \in \mathbb{R}^{n_k}$ and $X_k \in \mathbb{R}^{n_k \times p}$ denote the outcome vector and predictor matrix for study $k$, respectively. The linear mixed effects data generating model is of the form 
\begin{equation}
\label{eqn:lme}
     Y_k =  X_k\beta  +  Z_k  \gamma_k +  \epsilon_k, \quad k = 1, \ldots, K + V
\end{equation}
where $\beta \in \mathbb{R}^p$ are the fixed effects and $\gamma_k \in \mathbb{R}^q$ the random effects with $E\left[\gamma_k\right] = 0$ and $Cov(\gamma_k) = \text{diag(}\sigma^2_1, \ldots, \sigma^2_q) \eqqcolon G$. If $\sigma^2_j > 0,$ then the effect of the $j$th predictor varies across studies; if $\sigma^2_j = 0$, then the predictor has the same effect in each study.  The matrix $Z_k \in \mathbb{R}^{n_k \times q}$ is a subset of $X_k$ that corresponds to the random effects, and $\epsilon_k$ are the residual errors where $E[\epsilon_k] = 0, Cov(\epsilon_k) = \sigma^2_{\epsilon} I$, and $Cov(\gamma_k, \epsilon_k) = 0.$ We consider an extension of (\ref{eqn:lme}) and assume the study data are generated under the mixed effects model of the form
\begin{equation}
\label{eqn:gme}
     Y_k =  f(X_k)  +  Z_k  \gamma_k +  \epsilon_k, \quad k = 1, \ldots, K+V
\end{equation}
where $f(\cdot)$ is a real-valued function. Compared to (\ref{eqn:lme}), the model in (\ref{eqn:gme}) provides more flexibility in fitting the mean function $E(Y_k)$. 

For any study $k$, we assume $Y_k$ is centered to have zero mean and $X_k$ standardized to have zero mean and unit $\ell_2$ norm, i.e., $\norm{X_{jk}}_2 = 1$ for $j = 1, \ldots, p,$ where $X_{jk} \in \mathbb{R}^N$ denotes the $j$th predictor in study $k$. Unless otherwise stated, we use $i \in \{1, \ldots, N\}$ to index the observations, $j \in \{1, \ldots, p\}$ the predictors, and $k \in \{1, \ldots, K + V\}$ the studies. For example,  $X_{ijk} \in \mathbb{R}$ is the value of the $j$th predictor for observation $i$ in study $k.$ We formally introduce boosting on the merged study $(Y, X)$ in the next section, but the formulation is the same for the $k$th study if one were to replace $(Y, X)$ with $(Y_k, X_k)$. In particular, we focus on boosting with linear learners due to its analytical tractability. We denote a linear learner as an operator $H: \mathbb{R}^N \rightarrow \mathbb{R}^N$ that maps the responses $Y$ to fitted values $\hat{Y}$. Examples of linear learners include ridge regression and more general projectors to a class of basis functions such as regression or smoothing splines. We denote the basis-expanded predictor matrix by $\tilde{X} \in \mathbb{R}^{N \times P}$ and the subset of predictors with random effects by $\tilde{Z} \in \mathbb{R}^{N \times Q}$. We define the basis-expanded predictor matrix as
$$\tilde{X} = \left[h(X_i) \quad \cdots \quad h(X_N)\right]^T \in \mathbb{R}^{N \times P},$$
where $$h(X_i) = \left(h_{11}(X_{i1}),\ldots, h_{U_11}(X_{i1}), \ldots, h_{1p}(X_{ip}), \ldots, h_{U_Pp}(X_{ip})\right) \in \mathbb{R}^P, \quad i = 1,\ldots, N$$
is the vector of $P = \sum_{p}U_p$ one-dimensional basis functions evaluated at the predictors $X_i \in \mathbb{R}^p$. As an example, suppose we have $p=2$ covariates, $X_{i1}, X_{i2}$, and we want to model $X_{i1}$ linearly and $X_{i2}$ with a cubic spline at knots $\xi_1 = 0$ and $\xi_2 = 1.5$. The basis-expanded predictor matrix $\tilde{X}$ contains the following vector of $P = 6$ basis functions:
$$h(X_i) = \left(h_{11}(X_{i1}), h_{12}(X_{i2}), h_{22}(X_{i2}), h_{32}(X_{i2}), h_{42}(X_{i2}), h_{51}(X_{i2})\right), \quad i = 1, \ldots, N$$ where
\begin{center}
    $
\begin{aligned}[c]
  h_{11}(X_{i1}) &= X_{i1}\\
 h_{12}(X_{i2}) &= X_{i2}\\
h_{22}(X_{i2}) &= X_{i2}^2\\
\end{aligned}
\qquad \qquad
\begin{aligned}[c]
h_{32}(X_{i2}) &= X_{i2}^3\\
h_{42}(X_{i2}) &= (X_{i2} - 0)^3_{+}\\
h_{52}(X_{i2}) &= (X_{i2} - 1.5)^3_+
\end{aligned}
$
\end{center}
and $(X_{i2} - \xi)^3_+ = max\left\{(X_{i2}-\xi)^3, 0\right\}.$ For $\lambda \geq 0$, our goal is to minimize the objective $$||Y - \tilde{X}\beta||^2_2 + \lambda \beta^T\beta$$
with respect to parameters $\beta \in \mathbb{R}^P$. We denote the vector of coefficient estimates and fitted values by $\hat{\beta} \coloneqq BY$ and $\hat{Y} \coloneqq HY$, respectively, where $$B\coloneqq (\tilde{X}^T\tilde{X} + \lambda I )^{-1}\tilde{X}^T \in \mathbb{R}^{P \times N}$$ and $$H\coloneqq \tilde{X}(\tilde{X}^T\tilde{X} + \lambda I )^{-1}\tilde{X}^T = \tilde{X}B \in \mathbb{R}^{N \times N}.$$

\subsection{Boosting with linear learners} 
Given the basis-expanded predictor matrix $\tilde{X} \in \mathbb{R}^{N \times P}$, the goal of boosting is to obtain an estimate $\hat{F}(\tilde{X})$ of the function $F(\tilde{X})$ that minimizes the expected loss 
$E\left[\ell(Y, F(\tilde{X}))\right]$ for a given loss function $\ell(\cdot, \cdot): \mathbb{R}^{N} \times \mathbb{R}^{N} \rightarrow \mathbb{R}^N_+,$. Here, the
 outcome $Y \in \mathbb{R}^N$ may be continuous (regression problem) or discrete (classification problem). Examples of $\ell(Y, F)$ include exponential loss $exp(YF)$ for AdaBoost (\cite{freund1995boosting}) and $\ell_2$ (squared error) loss $(Y - F)^2/2$ for $\ell_2$ boosting (\cite{buhlmann2003boosting}). In finite samples, estimation of $F(\cdot)$ is done by minimizing the empirical risk via functional gradient descent where the base learner $g(\tilde{X}; \hat{\theta})$ is repeatedly fit to the negative gradient vector $$r =\left. \frac{-\partial \ell(Y, F)}{\partial F}\right|_{F = \hat{F}_{(m)}(\tilde{X})}$$ evaluated at $\hat{F}_{(m)}(\tilde{X}) = \hat{F}_{(m-1)}(\tilde{X}) + \eta g(\tilde{X}; \hat{\theta}_m)$ across $m = 1, \ldots, M$ iterations. Here, $\eta \in (0, 1]$ denotes the learning rate, and $\hat{\theta}_m$ denotes the estimated finite or infinite-dimensional parameter that characterizes $g$ (i.e., if $g$ is a regression tree, then $\theta$ denotes the tree depth, minimum number of observations in a leaf, etc.). Under $\ell_2$ loss, the negative gradient at iteration $m$ is equivalent to the residuals $Y - \hat{F}_m(\tilde{X})$. Therefore, $\ell_2$ boosting produces a stage-wise approach that iteratively fits to the current residuals (\cite{buhlmann2003boosting, friedman2001greedy}).

Let $\hat{\beta}_{(m)} \in \mathbb{R}^P$ and $\hat{Y}_{(m)} \in \mathbb{R}^N$ denote the coefficient estimates and fitted values at the $m$th boosting iteration, respectively. We describe $\ell_2$ boosting with linear learners in \textbf{Algorithm 1}.
\begin{algorithm}
\caption{$\ell_2$ boosting with linear learners.}
\begin{algorithmic}[1]

\State Initialization:
$$\hat{ \beta}_{(0)} =  0, \quad \hat{ Y}_{(0)} = 0$$

\State Iteration: For $m = 1, 2, \ldots, M,$ fit a linear learner to the residuals $ r_{(m)} =  Y - \hat{ Y}_{(m-1)}$ and obtain the estimated coefficients $$\hat{ \beta}_{(m)}^{\text{current}} = Br_{(m)}$$
and fitted values
$$\hat{ Y}_{(m)}^{\text{current}} =  Hr_{(m)}.$$
The new coefficient estimates are given by:
$$\hat{ \beta}_{(m)} = \hat{ \beta}_{(m - 1)} + \eta \hat{ \beta}_{(m)}^{\text{current}}$$
The new fitted values are given by:
$$\hat{ Y}_{(m)} = \hat{ Y}_{(m - 1)} +\eta \hat{ Y}_{(m)}^{\text{current}}$$
where $\eta \in (0, 1]$ is the learning rate.
\end{algorithmic}
\end{algorithm}

By Proposition 1 in \cite{buhlmann2003boosting}, the $\ell_2$ boosting coefficient estimates at iteration $M$ can be written as:
\begin{equation}
\label{eqn:l2boost}
    \hat{\beta}^{\text{Merge}}_{(M)} = \sum_{m=1}^M \eta B(I - \eta H)^{m-1}Y.
\end{equation}
Equation (\ref{eqn:l2boost}) represents $\hat{\beta}^{\text{Merge}}_{(M)}$ as the sum across coefficient estimates obtained from repeatedly fitting a linear learner $H$ to residuals $r_{(m)} = (I - \eta H)^{m-1}Y$ at iteration $m = 1, \ldots, M.$ The ensemble estimator, based on pre-specified weights $w_k$ such that $\sum_{k=1}^K w_k = 1,$ is
\begin{equation}
 \hat{\beta}^{Ens}_{(M)} = \sum_{k=1}^K w_k \hat{\beta}_{k(M)} = \sum_{k=1}^K w_k \left[\sum_{m=1}^M \eta B_k(I - \eta H_k)^{m-1}Y_k\right]
\end{equation}
where $B_k$ and $H_k \hspace{0.3em} (k = 1, \ldots, K)$ are study-specific analogs of $B$ and $H,$ respectively. 

\subsection{Boosting with component-wise linear learners}
Boosting with component-wise linear learners (\cite{buhlmann2007boosting, buhlmann2003boosting}), also known as LS-Boost (\cite{friedman2001greedy}) or least squares boosting (\cite{freund2017new}), determines the predictor $\tilde{X}_{\hat{j}_{(m)}} \in \mathbb{R}^N$ that results in the maximal decrease in the univariate least squares fit to the current residuals $r_{(m)}$. The algorithm then updates the $\hat{j}_{(m)}$th coefficient and leaves the rest unchanged. Let $\hat{\beta}_{(m)j} \in \mathbb{R}$ denote the $j$th coefficient estimate at the $m$th iteration and $\hat{\beta}_{\hat{j}_{(m)}} \in \mathbb{R}$ the estimated coefficient of the selected covariate at iteration $m$. \textbf{Algorithm 2} describes boosting with component-wise linear learners.
\begin{algorithm}
\caption{$\ell_2$ boosting with component-wise linear learners.}
\begin{algorithmic}[1]

\State Initialization: $$\hat{ \beta}_{(0)} =  0, \quad \hat{ Y}_{(0)} =  0$$
\State Iteration: For $m = 1, 2, \ldots, M,$ compute the residuals
$$ r_{(m)} =  Y - \hat{ Y}_{(m-1)}.$$
Determine the covariate $\tilde{X}_{\hat{j}_{(m)}}$ that results in the best univariate least squares fit to $r_{(m)}:$
 $$\hat{j}_{(m)} = \argmin_{1 \leq j \leq P} \sum_{i=1}^N \left( r_{(m)i} -  \tilde{X}_{ij}\hat{\beta}_{(m)j}\right)^2.$$
Calculate the corresponding coefficient estimate:
$$
\hat{\beta}_{\hat{j}_{(m)}} =  \left(\tilde{X}_{\hat{j}_{(m)}}^T  \tilde{X}_{\hat{j}_{(m)}}\right)^{-1} \tilde{X}_{\hat{j}_{(m)}}^T r_{(m)}.$$
Update the fitted values and the coefficient estimate for the $\hat{j}_{(m)}$th covariate
\begin{align*}
 \hat{Y}_{(m)} &= \hat{ Y}_{(m - 1)} + \eta  \tilde{X}_{\hat{j}_{(m)}}\hat{\beta}_{\hat{j}_{(m)}}\\
\hat{\beta}_{(m)\hat{j}_{(m)}} &= \hat{\beta}_{(m-1)\hat{j}_{(m)}} + \eta \hat{\beta}_{\hat{j}_{(m)}}
\end{align*}
where $\eta \in (0, 1]$ is a learning rate.
\end{algorithmic}
\end{algorithm}

\begin{proposition} 
\label{prop:1}
Let $e_{\hat{j}_{(m)}} \in \mathbb{R}^P$ denote a unit vector with a 1 in the $\hat{j}_{(m)}$-th position,
$$B_{(m)} = e_{\hat{j}_{(m)}} \left(\tilde{X}_{\hat{j}_{(m)}}^T \tilde{X}_{\hat{j}_{(m)}}\right)^{-1}\tilde{X}^T_{\hat{j}_{(m)}},$$
and 
$$H_{(m)} = \tilde{X}_{\hat{j}_{(m)}}\left(\tilde{X}_{\hat{j}_{(m)}}^T \tilde{X}_{\hat{j}_{(m)}}\right)^{-1}\tilde{X}^T_{\hat{j}_{(m)}}.$$
The coefficient estimates for $\ell_2$ boosting with component-wise linear learners at iteration $M$ can be written as:
\begin{equation}
  \hat{\beta}^{\text{Merge, CW}}_{(M)} = \sum_{m=1}^M \eta B_{(m)}\left(\prod_{\ell = 0}^{m-1} \left(I - \eta H_{(m-\ell-1)}\right)\right)Y.
\end{equation}
\end{proposition} 
A proof is provided in the appendix. Proposition \ref{prop:1} represents $\hat{\beta}^{\text{Merge,CW}}_{(M)}$ as the sum across coefficient estimates obtained from repeatedly fitting a univariate linear learner $H_{(m)}$ to the current residuals $r_{(m)} =(\prod_{\ell = 0}^{m-1} (I - \eta H_{(m-\ell-1)}))Y$ at iteration $m$. As $M \rightarrow \infty,$ $\hat{\beta}^{\text{Merge,CW}}_{(M)}$ converges to a least squares solution that is unique if the predictor matrix has full rank (\cite{buhlmann2007boosting}). The ensemble estimator, based on pre-specified weights $w_k$, is

\begin{equation}
    \hat{\beta}^{\text{Ens, CW}}_{(M)} = \sum_{k=1}^K w_k \hat{\beta}^{\text{CW}}_{(M)k} = \sum_{k=1}^K w_k \left[\sum_{m=1}^M \eta  B_{(m)k} \left(\prod_{\ell = 0}^{m-1} \left( I - \eta  H_{(m-\ell-1)k}\right)\right) Y_k\right]
\end{equation}
where $B_{(m)k}$ and $H_{(m)k}$ are study-specific analogs of $B_{(m)}$ and $H_{(m)},$ respectively.
\subsection{Performance comparison}\label{performance}
We compare merging and ensembling based on mean squared prediction error (MSPE) of $V$ unseen test studies $\tilde{X}_0 \in \mathbb{R}^{N^{\text{Test}} \times P}$ with unknown outcome vector $Y_0 \in \mathbb{R}^{N^{\text{Test}}}$, $$E[||Y_0 - \tilde{X}_0\hat{\beta}_{(M)}||^2_2]$$  where $\norm{\cdot}_2$ denotes the $\ell_2$ norm. 
To properly characterize the performance of boosting with component-wise linear learners (\textbf{Algorithm 2}), we account for the algorithm's adaptive nature by conditioning on its selection path. To make progress analytically, we assume $Y$ is normally distributed with mean $\mu \coloneqq f(\tilde{X})$ and covariance $\Sigma \coloneqq \text{blkdiag}(\{Z_kGZ_k^T+\sigma^2_{\epsilon}I\}^K_{k=1})$. Note that at iteration $m$, the covariate $\Tilde{X}_{\hat{j}_{(m)}}$ will result in the best univariate least squares fit to $r_{(m)}$ if and only if it satisfies 
$$
     \norm{(I -  H_{(\hat{j}_{(m)})}) r_{(m)}}_2^2 \leq \norm{( I -  H_{(j)}) r_{(m)}}_2^2,$$
     which is equivalent to
     \begin{equation}
     \label{eqn:selection}
         (sgn_{(m)} \tilde{X}_{j(m)}^T/\norm{ \tilde{X}_{\hat{j}_{(m)}}}_2 \pm   \tilde{X}_j^T/\norm{ \tilde{X}_{j}}_2) r^{(m)} \geq 0
     \end{equation}
$\forall j \neq \hat{j}_{(m)}$, $sgn_{(m)} = \text{sign}(\tilde{X}_{\hat{j}_{(m)}}^T r_{(m)})$, where $$r_{(m)} = \prod_{\ell=0}^{m-1}( I - \eta  H_{(m-\ell-1)}) Y :=  \Upsilon_{(m)}  Y.$$ With fixed $\tilde{X}$, the inequalities in (\ref{eqn:selection}) can be compactly represented as the polyhedral representation $\Gamma Y \geq 0$ for a matrix $\Gamma \in \mathbb{R}^{2M(P-1) \times N},$ with the $(\tilde{m} + 2(j - \omega(j)) - 1)$th and $(\tilde{m} + 2(j - \omega(j)))$th rows given by $$(sgn_{(m)} \tilde{X}_{\hat{j}_{(m)}}^T/\norm{ \tilde{X}_{\hat{j}_{(m)}}}_2 \pm \tilde{X}_j^T/\norm{ \tilde{X}_j}_2) \Upsilon^{(m)}$$ $\forall j \neq \hat{j}_{(m)}$ with $\tilde{m} = 2(P-1)(m-1)$ and $\omega(j) = \mathbb{1}\{j > \hat{j}_{(m)}\}$ (\cite{rugamer2020inference}). The $j$th regression coefficient in \textbf{Algorithm 2} can be written as $$\hat{\beta}^{\text{Merge, CW}}_{(M)j} = v_j^TY,$$
where $$v_j = (\sum_{m=1}^M \eta B_{(m)} (\prod_{\ell = 0}^{m-1}(I - \eta H_{(m-\ell - 1)}))^Te_j,$$ and $e_j \in \mathbb{R}^P$ is a unit vector. The distribution of $\hat{\beta}^{\text{Merge, CW}}_{(M)j}$ conditional on the selection path is given by the polyhedral lemma in \cite{lee2016exact}.

\begin{lemma}[Polyhedral lemma from \cite{lee2016exact}]
\label{lemma:lee}
Given the selection path $$\mathcal{P}\coloneqq \{Y:\Gamma Y \geq 0, z_j = z\},$$ where $z_j \coloneqq (I - c_jv_j^T)Y$ and $c_j \coloneqq \Sigma v_j(v_j^T\Sigma v_j)^{-1}$, $$\hat{\beta}^{\text{Merge, CW}}_{(M)j}|\mathcal{P} \sim \text{TruncatedNormal}\left(v_j^T\mu, v_j\Sigma v_j^T, a_j, b_j\right),$$
where
 \begin{align*}
    a_j &= \max_{\ell:(\Gamma c_j)_{\ell} > 0} \frac{0-(\Gamma z_j)_{\ell}}{(\Gamma c_j)_{\ell}}\\
    b_j &= \min_{\ell: (\Gamma c_j)_{\ell} < 0} \frac{0-(\Gamma z_j)_{\ell}}{(\Gamma c_j)_{\ell}}.
\end{align*}
\end{lemma}
A proof is provided in the appendix. The conditioning is important because it properly accounts for the adaptive nature of \textbf{Algorithm 2}.  Conceptually, it measures the magnitude of $\hat{\beta}^{\text{Merge, CW}}_{(M)j}$ among random vectors $Y$ that would result in the selection path $\Gamma Y \geq 0$ for a fixed value of $z_j$. When $\Sigma = \sigma^2I,$ $z_j = (I - v_j(v_j^Tv_j)^{-1}v_j^T)Y$ is the projection onto the orthocomplement of $v_j.$ Accordingly, the polyhedron $\Gamma Y \geq 0$ holds if and only if $v_j^TY$ does not deviate too far from $z_j,$ hence trapping it between bounds $a_j$ and $b_j$ (\cite{tibshirani2016exact}). Moreover, because $a_j$ and $b_j$ are functions of $z_j$ alone, they are independent of $v^TY$ under normality. The result in Lemma \ref{lemma:lee} allows us to analytically characterize the mean squared error of the estimators $\hat{\beta}^{\text{Merge, CW}}_{(M)j}$ and $\hat{\beta}^{\text{Ens, CW}}_{(M)j}$ conditional on their respective selection paths.
\subsection{Implicit regularization and early stopping} \label{subsection:earlystop} In \textbf{Algorithm 1} and \textbf{Algorithm 2}, the learning rate $\eta$ and stopping iteration $M$ together control the amount of shrinkage and training error. A smaller learning rate $\eta$ leads to slower overfitting but requires a larger $M$ to reduce the training error to zero. With a small $\eta$, it is possible to explore a larger class of models, which often leads to models with better predictive performance (\cite{friedman2001greedy}). While boosting algorithms are known to exhibit slow overfitting behavior with small values of $\eta$, it is necessary to implement early stopping strategies to avoid overfitting (\cite{schapire1998boosting}). The boosting fit for \textbf{Algorithm 1} in iteration $m$ (assuming $\eta = 1)$ is $$\mathcal{B}_{(m)}Y \coloneqq (I-(I-H)^{m+1})Y,$$
where $\mathcal{B}_{(m)}: \mathbb{R}^N \rightarrow \mathbb{R}^N$ is the boosting operator. For a base learner that satisfies $\norm{I - H} \leq 1$ for a suitable norm, we have $\mathcal{B}_{(m)}Y \rightarrow Y$ as $m \rightarrow \infty$. That is, if left to run forever, the boosting algorithm converges to the fully saturated  model $Y$ (\cite{buhlmann2007boosting}). A similar argument can be made for \textbf{Algorithm 2} where $$\mathcal{B}^{\text{CW}}_{(m)} = I-(I - H_{(\hat{j}_m)})(I - H_{(\hat{j}_{m-1})}) \cdots (I - H_{(\hat{j}_{1})})$$
is the component-wise boosting operator. We define the degrees of freedom at iteration $m$ as $tr(\mathcal{B}_{(m)})$ and use the corrected AIC criterion ($AIC_c$) (\cite{buehlmann2006boosting}) to choose the stopping iteration $M.$ Compared to cross-validation (CV), $AIC_c$-tuning is computationally efficient as it does not require running the boosting algorithm multiple times. For \textbf{Algorithm 1}, the $AIC_c$ at iteration $m$ is given by
\begin{equation}
    \label{eqn:AICC}
    AIC_c(m) = \log(\hat{\underline{\sigma}}^2) + \frac{1 + tr(\mathcal{B}_{(m)})/N}{1 - (tr(\mathcal{B}_{(m)}) + 2)/N},
\end{equation}
where $\underline{\hat{\sigma}}^2 = \frac{1}{N}\sum_{i=1}^N (Y_i - (\mathcal{B}_{(m)}Y)_i )^2$. The stopping iteration is $$M = \argmin_{1 \leq m \leq m_{upp}} AIC_c(m),$$ where $m_{upp}$ is a large upper bound for the candidate number of boosting iterations (\cite{buehlmann2006boosting}). For \textbf{Algorithm 2}, the $AIC_c$ is computed by replacing $\mathcal{B}_{(m)}$ with $\mathcal{B}^{\text{CW}}_{(m)}.$ We allow the stopping iterations to differ between the merged and ensemble learners. In our results, we denote them by $M$ and $M_{\text{Ens}} = \{M_k\}_{k=1}^K,$ respectively. 

\section{Results}
We summarize the degree of heterogeneity in predictor-outcome relationships across studies by the sum of the variances of the random effects divided by the number of fixed effects: $\overline{\sigma^2} \coloneqq tr(G)/P$, where $G \in \mathbb{R}^{Q \times Q}$. For boosting with linear learners, let $\tilde{R} = \sum_{m=1}^{M} \eta B(I-\eta H)^{m-1}$ and $\tilde{R}_k = \sum_{m=1}^{M} \eta B_k(I-\eta H_k)^{m-1}$. Let $b_{\text{Merge}} = Bias(\hat{\beta}^{\text{Merge}}_{(M)}) = \tilde{R}f(\tilde{X}) - f(\tilde{X}_0)$ denote the bias of the boosting coefficients for the merged estimator and $b_{\text{Ens}} = Bias(\hat{\beta}^{\text{Ens}}_{(M_{\text{Ens}})}) = \sum_{k=1}^K w_k \tilde{R}_kf(\tilde{X}_k) - f(\tilde{X}_0)$ the bias for the ensemble estimator. Let $Z' = blkdiag(\{Z_k\}_{k=1}^K)$ and $G' = blkdiag(\{G_k\}_{k=1}^K)$ where $G_k = G$ for $k = 1, \ldots, K.$

\subsection{Boosting with linear learners}
\begin{customthm}{1}\label{thm:1}
Suppose
\begin{equation}
\label{eq:thm1cond}
    \text{tr}( Z'^T  \tilde{R}^T \tilde{X}_0^T \tilde{X}_0  \tilde{R}  Z') - \sum_{k=1}^K w_k^2 \text{tr}( Z_k^T  \tilde{R}_k^T \tilde{X}_0^T \tilde{X}_0 \tilde{ R}_k  Z_k) > 0
\end{equation}

Define
\small
\begin{equation}
\label{eqn:tau}
\scriptsize
  \tau  =\frac{Q}{P} \times  \frac{\sigma^2_{\epsilon}(\sum_{k=1}^Kw_k^2 \text{\text{tr}}( \tilde{R}_k^T  \tilde{X}_0^T \tilde{X}_0  \tilde{R}_k)-\text{tr}( \tilde{R}^T \tilde{X}_0^T \tilde{X}_0  \tilde{R}))+ ( b^{\text{Ens}})^T b^{\text{Ens}} -  ( b^{\text{Merge}})^T b^{\text{Merge}}}{\text{tr}( Z'^T  \tilde{R}^T \tilde{X}_0^T \tilde{X}_0  \tilde{R}  Z') - \sum_{k=1}^K w_k^2 \text{tr}( Z_k^T  \tilde{R}_k^T \tilde{X}_0^T \tilde{X}_0  \tilde{R}_k  Z_k)}
\end{equation}
Then $E[\norm{Y_0 - \tilde{X}_0\hat{\beta}^{\text{Ens}}_{(M_{\text{Ens}})}}^2_2] \leq [\norm{Y_0 - \tilde{X}_0\hat{\beta}^{\text{Merge}}_{(M)}}^2_2]$ if and only if $\overline{\sigma}^2 \geq \tau.$
\end{customthm} 

A proof is provided in the appendix. Under the equal variances assumption, Theorem \ref{thm:1} characterizes a transition point $\tau$ beyond which ensembling outperforms merging for \textbf{Algorithm 1}. $\tau$ is characterized by differences in the predictive performance of merging vs. ensembling driven by within-study variability and bias in the numerator and between-study variability in the denominator. The condition in (\ref{eq:thm1cond}), which ensures $\tau$ is well defined, holds when the between-study variability of $\hat{\beta}^{\text{Merge}}_{(M)}$ is greater than that of $\hat{\beta}^{\text{Ens}}_{(M_{Ens})}$. This is generally true because merging does not account for between-study heterogeneity. $\tau$ depends on the population mean function $f$ through the bias term. Therefore, an estimate of $f$ is required to estimate the transition point unless the bias is equal to zero. One example of an unbiased estimator is ordinary least squares, which can be obtained by setting $H = \tilde{X}(\tilde{X}^T\tilde{X})\tilde{X}^T$ and $M = \eta = 1$. In general, for any linear learner $H: \mathbb{R}^N \rightarrow \mathbb{R}^N$, the transition point in \cite{guan2019merging} (cf., Theorem 1) is a special case of (\ref{eqn:tau}) when $M = \eta = 1$.

\begin{corollary}
\label{cor:1}
  Suppose $\text{tr}( Z'^T  \tilde{R}^T \tilde{X}_0^T \tilde{X}_0  \tilde{R}  Z') \neq 0.$ As $\sigma^2 \rightarrow \infty,$
  $$\frac{E[\norm{Y_0 - \tilde{X}\hat{ \beta}^{\text{Ens}}_{(M_{\text{Ens}})}}^2_2]}{ E[\norm{Y_0 - \tilde{X}_0\hat{ \beta}^{\text{Merge}}_{(M)}}^2_2]} \longrightarrow \frac{\sum_{k=1}^K w_k^2 \text{tr}( Z_k^T  \tilde{R}_k^T \tilde{X}_0^T \tilde{X}_0  \tilde{R}_k  Z_k)}{ \text{tr}( Z'^T  \tilde{R}^T \tilde{X}_0^T \tilde{X}_0  \tilde{R}  Z')}.$$
\end{corollary}
This result follows immediately from Theorem \ref{thm:1}. According to Corollary \ref{cor:1}, the asymptote of the MSPE ratio comparing ensembling to merging equals the ratio of between-study variability. Because the merged estimator does not account for between-study variability, the asymptote is less than one. 

Let $\sigma^2_{(1)}, \ldots, \sigma^2_{(D)}$ denote the distinct values of variances of the random effects where $D \leq Q$, and let $J_d$ denote the number of random effects with variance $\sigma^2_{(d)}.$
\begin{theorem}
\label{thm:2}
Suppose $$\max_d \sum_{i: \sigma^2_i = \sigma^2_{(d)}} \left[ \sum_{k=1}^K  \left(Z'^T\tilde{R}^T\tilde{X}_0^T\tilde{X}_0 \tilde{R}Z'\right)_{i + Q \times (k - 1), i + Q \times (k - 1)} -  w_k^2\left(Z_k^T\tilde{R}_k^T\tilde{X}^T_0\tilde{X}_0\tilde{R}_k Z_k\right)_{i,i}\right] > 0$$
and define 
\begin{equation}
    \label{eqn:tau1}
    \scriptsize
    \tau_1 = \frac{ \sigma^2_{\epsilon}(\sum_{k=1}^Kw_k^2 \text{\text{tr}}(\tilde{R}_k^T  \tilde{X}_0^T \tilde{X}_0  \tilde{R}_k)-\text{tr}( \tilde{R}^T \tilde{X}_0^T \tilde{X}_0  \tilde{R}))+( b^{\text{Ens}})^T b^{\text{Ens}} - ( b^{\text{Merge}})^T b^{\text{Merge}}}{P \max_{d} \frac{1}{J_d}\sum_{i: \sigma^2_i = \sigma^2_{(d)}} [ \sum_{k=1}^K  (Z'^T\tilde{R}^T\tilde{X}_0^T\tilde{X}_0 \tilde{R}Z')_{i + Q \times (k - 1), i + Q \times (k - 1)} -  w_k^2(Z_k^T\tilde{R}_k^T\tilde{X}^T_0\tilde{X}_0\tilde{R}_k Z_k)_{i,i}]}.
\end{equation}
Then $E[||Y_0 - \tilde{X}_0 \hat{\beta}^{\text{Ens}}_{(M_{\text{Ens}})}||_2^2] \geq E[||Y_0 - \tilde{X}_0 \hat{\beta}^{\text{Merge}}_{(M)}||_2^2]$ when $\overline{\sigma}^2 \leq \tau_1.$

Suppose $$\min_d \sum_{i: \sigma^2_i = \sigma^2_{(d)}} \left[ \sum_{k=1}^K  \left(Z'^T\tilde{R}^T\tilde{X}_0^T\tilde{X}_0 \tilde{R}Z'\right)_{i + Q \times (k - 1), i + Q \times (k - 1)} -  w_k^2\left(Z_k^T\tilde{R}_k^T\tilde{X}^T_0\tilde{X}_0\tilde{R}_k Z_k\right)_{i,i}\right] > 0$$
and define 
\begin{equation}
\label{eqn:tau2}
\scriptsize
     \tau_2 = \frac{ \sigma^2_{\epsilon}(\sum_{k=1}^Kw_k^2 \text{\text{tr}}( \tilde{R}_k^T  \tilde{X}_0^T \tilde{X}_0  \tilde{R}_k)-\text{tr}( \tilde{R}^T \tilde{X}_0^T \tilde{X}_0  \tilde{R}))+( b^{\text{Ens}})^T b^{\text{Ens}} - ( b^{\text{Merge}})^T b^{\text{Merge}}}{P \min_{d} \frac{1}{J_d}\sum_{i: \sigma^2_i = \sigma^2_{(d)}} [ \sum_{k=1}^K  (Z'^T\tilde{R}^T\tilde{X}_0^T\tilde{X}_0 \tilde{R}Z')_{i + Q \times (k - 1), i + Q \times (k - 1)} -  w_k^2(Z_k^T\tilde{R}_k^T\tilde{X}^T_0\tilde{X}_0\tilde{R}_k Z_k)_{i,i}]}.
\end{equation}
Then $E[||Y_0 - \tilde{X}_0 \hat{\beta}^{\text{Ens}}_{(M_{\text{Ens}})}||_2^2] \leq E[||Y_0 - \tilde{X}_0 \hat{\beta}^{\text{Merge}}_{(M)}||_2^2]$ when $\overline{\sigma}^2 \geq \tau_2.$
\end{theorem}
A proof is provided in the appendix. Theorem \ref{thm:2} generalizes Theorem \ref{thm:1} to account for unequal variances along the diagonal of $G$. It characterizes a transition interval $[\tau_1, \tau_2]$ where merging outperforms ensembling when $\overline{\sigma}^2 \leq \tau_1$ and vice versa when $\overline{\sigma}^2 \geq \tau_2.$ The transition interval provided by \cite{guan2019merging} (cf. Theorem 2) is a special case of (\ref{eqn:tau1}, \ref{eqn:tau2}) when $M = \eta = 1$.
\subsection{Boosting with component-wise linear learners}
\label{sec:cwboost}
 To properly characterize the performance of the boosting estimator in \textbf{Algorithm 2}, we condition on its selection path. To this end, we provide the conditional MSE of the merged and ensemble estimators in Proposition \ref{prop:2}. Assuming $Y \sim MVN(\mu, \Sigma)$, it follows that $Y_k$ is normal with mean $\mu_k \coloneqq f(\tilde{X}_k)$ and covariance $\Sigma_k \coloneqq Z_kGZ_k^T + \sigma^2_{\epsilon}I$ for $k = 1,\ldots, K$. Let $$\mathcal{P} = \{Y: \Gamma Y \geq 0, z_j = z\}$$ and $$\mathcal{P}^{\text{Ens}} = \{\mathcal{P}_1, \ldots, \mathcal{P}_K\}$$ denote the conditioning events for the merged and ensemble estimators, respectively, where $$\mathcal{P}_k \coloneqq \{Y_k: \Gamma_k Y_k \geq 0, z_{jk} = z_k\}$$ summarizes the boosting path from fitting \textbf{Algorithm 2} to the data in study $k$.  Let $\bar{\mu}_j = v_j^T\mu$ and $\vartheta^2_j = v_j\Sigma v_j^T$ denote the mean and variance of $\hat{\beta}^{\text{CW, Merge}}_{(M)j}=v^T_jY$, respectively. And let $\alpha_j = \frac{a_j - \bar{\mu}_j}{\vartheta_j}$ and $\xi_j = \frac{b_j - \bar{\mu}_j}{\vartheta_j}$ denote the standardized lower and upper truncation limits. We denote the study-specific versions of $\bar{\mu}_j, \theta_j, \alpha_j$ and $\xi_j$ by $\bar{\mu}_{jk}, \theta_{jk}, \alpha_{jk},$ and $\xi_{jk},$ respectively.
\begin{proposition} Let $\phi(\cdot)$ and $\Phi(\cdot)$ denote the probability density and cumulative distribution functions of a standard normal variable, respectively. The conditional mean squared error (MSE) of the merged estimator is
\label{prop:2}
{\scriptsize 
\begin{align*}
     E\left[\left.\left(\hat{\beta}^{\text{Merge, CW}}_{(M)j} - \beta_j\right)^2\right|\mathcal{P}\right] &= \left( \bar{\mu}_j - \vartheta_j\left(\frac{\phi(\xi_j)-\phi(\alpha_j)}{\Phi(\xi_j) - \Phi(\alpha_j)}\right) - \beta_j\right)^2 \\
     &+ \vartheta^2_j \left(1 - \frac{\xi_j\phi(\xi_j) - \alpha_j\phi(\alpha_j)}{\Phi(\xi_j) - \Phi(\alpha_j)} - \left(\frac{\phi(\xi_j) - \phi(\alpha_j)}{\Phi(\xi_j) - \Phi(\alpha_j)}\right)^2\right).
\end{align*}
}
The conditional MSE of the ensemble estimator is
{\scriptsize
\begin{align*}
    E\left[\left.\left(\hat{\beta}^{\text{Ens, CW}}_{(M_{\text{Ens}})j} - \beta_j\right)^2\right|\mathcal{P}^{\text{Ens}}\right] &=
    \left(\sum_{k=1}^K w_k  \left(\bar{\mu}_{jk} - \vartheta_{jk}\left(\frac{\phi(\xi_{jk})-\phi(\alpha_{jk})}{\Phi(\xi_{jk}) - \Phi(\alpha_{jk})}\right)\right) - \beta_j\right)^2\\
  &+   \sum_{k=1}^K w_k^2 \vartheta^2_{jk} \left(1 - \frac{\xi_{jk}\phi(\xi_{jk}) - \alpha_{jk}\phi(\alpha_{jk})}{\Phi(\xi_{jk}) - \Phi(\alpha_{jk})} - \left(\frac{\phi(\xi_{jk}) - \phi(\alpha_{jk})}{\Phi(\xi_{jk}) - \Phi(\alpha_{jk})}\right)^2\right).\\ 
\end{align*}}
\end{proposition}
A proof is provided in the appendix. Proposition \ref{prop:2} characterizes the conditional MSE of boosting estimators via the bias-variance decomposition. By the polyhedral lemma (\cite{lee2016exact}), the selection path $\Gamma Y \geq 0$ is equivalent to truncating $\hat{\beta}^{\text{Merge, CW}}_{(M)j} = v_j^TY$ to an interval $[a_j, b_j]$ around $z_j$. When there is no between-study heterogeneity, $z_j = (I-v_j(v_j^Tv_j)^{-1}v^T)Y$ is the residual from projecting $Y$ onto $v_j$. Loosely speaking, the selection path is equivalent to $v_j^TY$ not deviating too far from $z_j$. As shown in Section \ref{performance}, we can rewrite the selection path as a system of $2M(P-1)$ inequalities with the variable $v_j^TY$:
\begin{equation}
\label{eqn:linearSys}
    \{\Gamma Y \geq 0\} = \{\Gamma c_j (v_j^TY) \leq - \Gamma z_j\}.
\end{equation}
For fixed $P$, as the number of boosting iterations $M$ increases, the number of linear inequalities (or constraints) in (\ref{eqn:linearSys}) also increases; as a result, the size of the polyhedron $\Gamma Y \geq 0$ decreases. A smaller polyhedron generally leads to a narrower truncation interval $[a_j, b_j]$ around $v_j^TY$. Intuitively, a tighter truncation interval leads to reduced variance. When between-study heterogeneity is low, at a fixed learning rate $\eta$, the merged model generally requires a later stopping iteration than the study-specific model due to the increase in sample size. Therefore, $\hat{\beta}^{\text{Merge, CW}}_{(M)j}$ tends to have a tighter truncation region, and as a result, smaller variance than $\hat{\beta}^{\text{Ens, CW}}_{(M_{\text{Ens}})j}$. As between-study heterogeneity increases, the merged model often has an earlier stopping iteration to avoid overfitting, so $Var(\hat{\beta}^{\text{Merge, CW}}_{(M)j}) > Var(\hat{\beta}^{\text{Ens, CW}}_{(M_{\text{Ens}})j})$. In practice, the variance component in Proposition \ref{prop:2} can be computed given estimates of $\sigma^2$ and $f$.

\section{Simulations}
We conducted simulations to evaluate the performance of boosting with four base learners: ridge, component-wise least squares (CW-LS), component-wise cubic smoothing splines (CW-CS) and regression trees. We sampled predictors from the \texttt{curatedOvarianData} R package (\cite{ganzfried2013curatedovariandata}) to reflect realistic and potentially heterogeneous predictor distributions. The true data-generating model contains $p=10$ predictors of which $5$ have random effects. The outcome for individual $i$ in study $k$ is 
\begin{equation}
    Y_{ik} = f(X_{ik}) + Z_{ik} \gamma_k + \epsilon_{ik},
\end{equation}
where $\gamma_k \sim MVN(0, G)$ with $G = diag(\sigma^2_1, \ldots, \sigma^2_5)$, $Z_{ik} = (X_{3ik}, X_{4ik}, X_{5ik}, X_{6ik}, X_{7ik})$, and $\epsilon_{ik} \sim N(0, \sigma^2_{\epsilon})$ with $\sigma^2_{\epsilon} = 1$ for $i = 1, \ldots, n_k, k = 1, \ldots, K.$ The mean function $f$ has the form
\begin{small}
\begin{align}
 \label{eq:simfun}
     f(X_{ik}) &= -0.28 h_{11}(X_{1ik}) -0.12 h_{21}(X_{1ik}) -0.78 h_{31}(X_{1ik}) +0.035 h_{41}(X_{1ik}) -0.23X_{2ik} \nonumber\\
     &+1.56 X_{3ik} -0.0056 X_{4ik} + 0.13 X_{5ik} +0.0013 X_{6ik} - 0.00071 X_{7ik} - 0.0023 X_{8ik} \nonumber\\
     &-0.69 X_{9ik} + 0.016 X_{10ik}
\end{align}
\end{small}
where $h_{11}, \ldots, h_{41}$ are cubic basis splines with a knot at 0, and the coefficients were generated from $N(0, 0.5)$. The coefficients for $X_{2ik}, X_{3ik}, X_{5ij}$ and $X_{9ik}$ were generated from $N(0, 1)$, and those for $X_{4ik}, X_{6ik}, X_{7ik}, X_{8ik}, $ and $X_{10ik}$ were generated from $N(0, 0.01).$

We generated $K = 4$ training and $V = 4$ test studies of size $100$. For each simulation replicate $s = 1, \ldots, 500$, we generated outcomes for varying levels of $\overline{\sigma}^2$, trained merged and multi-study ensemble boosting models and evaluated them on the test studies. The outcome was centered to have zero mean, and predictors were standardized to have zero mean and unit $\ell_2$ norm. The regularization parameter $\lambda$ for ridge boosting and stopping iteration $M$ for tree boosting were chosen using 3-fold cross validation. The stopping iteration for linear base learners (ridge, CW-LS, and CW-CS) were chosen based on the $AIC_c$-tuning procedure described in Section \ref{subsection:earlystop}. All hyperparameters were tuned on a held-out data set of size 400 with $\sigma^2$ set to zero. For tree boosting, we set the maximum tree-depth to two. A learning rate of $\eta = 0.5$ was used for all boosting models. For the ensemble estimator, equal weight was assigned to each study. We considered two cases for the structure of $G:$ 1) equal variance and 2) unequal variance. In the first case, Figure \ref{fig:1} shows the relative predictive performance comparing multi-study ensembling to merging for varying levels of $\overline{\sigma}^2$. When $\overline{\sigma}^2$ was small, the merged learner outperformed the ensemble learner. As $\overline{\sigma}^2$ increased, there exists a transition point beyond which ensembling outperformed merging. The empirical transition point based on simulation results confirmed the theoretical transition point (\ref{eqn:tau}) for boosting with linear learners. As $\overline{\sigma}^2$ tended to infinity, the log relative performance ratio tended to $-0.81$ by Corollary \ref{cor:1}. Figure \ref{fig:2} shows the relative predictive performance under the unequal variance case. For boosting with linear learners, there exists a transition interval $[\tau_1, \tau_2]$ where merging outperformed ensembling when $\overline{\sigma}^2 \leq \tau_1$ and vice versa when $\overline{\sigma}^2 \geq \tau_2.$ Compared to boosting with linear or tree learners, boosting with component-wise learners had an earlier transition point.   

For boosting with component-wise linear learners, we compared the performance of merging and multi-study ensembling based on results in Proposition \ref{prop:2}. In each simulation replicate, we generated outcomes based on (\ref{eq:simfun}) and estimated $\beta^{\text{CW, Merge}}_{(M)}$ and  $\beta^{\text{CW, Ens}}_{(M)}$ with $M$ set to 30. We assumed equal variance along the diagonal of $G$. At each boosting iteration $m= 1,...,M$, we evaluated the MSE for both estimators with respect to $\beta_6= 1.72$ conditional on the boosting path up to iteration $m$. We chose to evaluate the coefficient associated with $X_6$ because the true data-generating coefficient $\beta_6$ had the largest magnitude, and as a result, the component-wise boosting algorithm was more likely to select $X_6$. Figure \ref{fig:3} shows the MSE associated with the merged and ensemble estimators at $\overline{\sigma}^2= 0.01$ and 0.05. We chose these values because the empirical transition point for boosting with component-wise linear learners in Figure \ref{fig:1} lies between 0.01 and 0.05. When $\overline{\sigma}^2$ = 0.01, merging outperformed ensembling. As the number of boosting iterations increased, both performed similarly. At $\overline{\sigma}^2 = 0.05,$ merging outperformed ensembling up until $M = 20$, beyond which ensembling began to show preferable performance.

\section{Breast Cancer Application}
Using data from the \texttt{curatedBreastData} R package (\cite{curatedBreastData}), we illustrated how the transition point theory could guide decisions on merging vs. ensembling.  This R package contains 34 high-quality gene expression microarray studies from over 16 clinical trials on individuals with breast cancer. The studies were normalized and  post-processed using the \texttt{processExpressionSetList()} function. In practice, a key determinant of breast cancer prognosis and staging is tumor size (\cite{fleming1997ajcc}). Clinicians use the TNM (tumor, node, metastasis) system to describe how extensive the breast cancer is. Under this system, "T" plus a letter or number (0 to 4) is used to describe the size (in centimeters (cm)) and location of the tumor. While the best way to measure the tumor is after it has been removed from the breast, information on tumor size can help clinicians develop effective treatment strategies. Common treatment options for breast cancer include surgery (e.g., mastectomy or lumpectomy), drug therapy (e.g., chemotherapy or immunotherapy) or a combination of both (\cite{gradishar2021nccn}). 

In our data illustration, the goal was to predict tumor size (cm) before treatment and surgery. We trained boosting models on $K = 5$ training studies with a combined size of $N = 643$: ID 1379 ($n = 60$), ID 2034 ($n = 281$), ID 9893 ($n = 155$), ID 19615 ($n = 115$) and ID 21974 ($n = 32$) and evaluated them on $V = 4$ test studies with a combined size of $N^{\text{Test}} = 366$: ID 21997 ($n = 94$), ID 22226 ($n = 144$), ID 22358 ($n = 122$), and ID 33658 ($n = 10$).  We selected the top $p = 40$ gene markers that were most highly correlated with tumor size in the training studies as predictors and randomly selected $q=8$ to have random effects with unequal variance. To calculate the transition interval from Theorem \ref{thm:2}, we trained boosting models with ridge learners using two strategies: merging and ensembling. We also estimated the variances of the random effects ($\sigma^2_1, \ldots, \sigma^2_8)$ and residual error ($\sigma^2_{\epsilon}$) by fitting a linear mixed effects model using restricted maximum likelihood. The estimate of $\overline{\sigma}^2$ and $\sigma^2_{\epsilon}$ were $4.32\times 10^{-2}$ and 1.053, respectively, and the transition interval was $[0.020, 0.026]$. In addition to ridge regression, we trained boosting models with three other base learners: CW-LS, CW-CS and regression trees. Results comparing the predictive performance of ensembling vs. merging are shown in Figure \ref{fig:4}. By Theorem \ref{thm:2}, merging would be preferred over ensembling for boosting with ridge learners because the estimate of $\overline{\sigma}^2$ was smaller than the lower bound of the transition interval. This result was corroborated by the boxplot of performance ratios in Figure \ref{fig:4}. 

Among the boosting algorithms that perform variable selection, ensembling outperformed merging when boosting with regression trees, and both performed similarly when boosting with component-wise learners. Table \ref{tab:1} summarizes the top three genes selected by each algorithm. Genes were ordered by decreasing variable importance, which was defined as the reduction in training error attributable to selecting a particular gene. In the merged study, both boosting with CW-CS and trees selected the same three genes: \textit{S100P, MMP11,} and \textit{E2F8}, whereas boosting with CW-LS selected \textit{S100P, ASPN,} and \textit{STY1}. This may be attributed to the fact that, compared to CW-LS, CW-CS and trees are more flexible and can capture non-linear trends in the data. Overall, there was some overlap in the genes that were selected by the three base learners across studies. In study ID 1379, all three base learners selected \textit{S100P}, and all but the tree learner selected \textit{AEBP1}. In studies ID 9893, 19615 and 21974, all three learners selected \textit{PPP1R3C}, \textit{CD9}, and \textit{CD69}, respectively. Tree boosting selected a single gene in studies ID 1379, 9893, and 21974 because the optimal number of boosting iterations determined by 3-fold CV was one. In general, CV-tuning leads to earlier stopping iterations than $AIC_c$-tuning as CV approximates the test error on a smaller sample. 

\section{Discussion}
In this paper, we studied boosting algorithms in a regression setting and compared merging and multi-study ensembling for improving cross-study replicability of predictions. We assumed a flexible mixed effects model with potential heterogeneity in predictor-outcome relationships across studies and provided theoretical guidelines for determining whether it was more beneficial to merge or to ensemble. In particular, we extended the transition point theory from \cite{guan2019merging} to boosting with linear learners. For boosting with component-wise linear learners, we characterized a bias-variance decomposition of estimation error conditional on the selection path. 

Boosting under $\ell_2$ loss is computationally simple and analytically attractive. In general, performance of the algorithm is inextricably linked with the choice of learning rate $\eta$ and stopping iteration $M.$ Common tuning procedures include $AIC_c$ tuning, cross-validation, and restricting the total step size (\cite{zhang2005boosting}). When both $\eta$ and $M$ are set to one, the transition point results on boosting coincide with those on ordinary least squares and ridge regression from \cite{guan2019merging}. A smaller $\eta$ corresponds to increased shrinkage of the effect estimates and decreased complexity of the boosting fit. For fixed $M$, decreasing $\eta$ results in a smaller transition point $\tau$, suggesting that multi-study ensembling would be preferred over merging at a lower threshold of heterogeneity. This can be attributed to the fact that for a fixed $M$, merging would require a larger $\eta$ due to the increase in sample size. Because of the interplay between $\eta$ and $M$, for a fixed $\eta$, decreasing $M$ also leads to a smaller $\tau.$ \cite{buehlmann2006boosting} noted that a smaller $\eta$ resulted in a weaker learner with reduced variance, and this was empirically shown to be more successful than a strong learner. 

We focused on $\ell_2$ boosting with linear learners for the opportunity to pursue closed-form solutions. With an appropriate choice of basis function, these learners can in theory
approximate any sufficiently smooth function to any level of precision (\cite{stone1948generalized}). In our simulations, the empirical transition points of boosting with ridge learners and boosting with regression trees were similar, suggesting that in certain scenarios it may be reasonable to consider the transition point theory in Theorems \ref{thm:1} and \ref{thm:2} as a proxy when comparing merging and ensembling for boosted trees. It is important to note, however, that such an approximation may not be warranted in settings where the choice of hyperparameters differ from that of our simulations. Although this paper focuses on boosting algorithms, we acknowledge important connections with other machine learning methods. A close relative of boosting with component-wise linear learners is the incremental forward stagewise algorithm (FS), which selects the covariate most correlated (in absolute value) with the residuals $r_{(m)}$ (\cite{efron2004least}). Because the covariates are standardized, both algorithms lead to the same variable selection for a given $r_{(m)}$. 

A potential limitation of Theorems \ref{thm:1} and \ref{thm:2} is that the tuning parameters (e.g., $\eta$ and $M$) are treated as fixed. These quantities are typically chosen by tuning procedures that introduce additional variability. Although we assumed the same $\eta$ for merging and ensembling in simulations, the transition point $\tau$ can be estimated with different values of $\eta$, which may be more realistic in practice. For the ensembling approach, we assigned equal weight to each study, which is equivalent to averaging the predictions. The equal-weighting strategy is a special case of stacking (\cite{breiman1996stacked, ren2020cross}) and is preferred in settings where studies have similar sample sizes. 

Many areas of biomedical research face a replication crisis in which scientific studies are difficult or impossible to replicate (\cite{ioannidis2005most}). An equally important but less commonly examined issue is the replicability of prediction models. To improve cross-study replicability of predictions, our work provides a theoretical rationale for choosing multi-study ensembling over merging when between-study heterogeneity exceeds a well-defined threshold. As many areas of science are becoming data-rich, it is critical to simultaneously consider and systematically integrate multiple studies to improve cross-study replicability of predictions. 

\clearpage

\begin{section}{Tables and Figures}

\begin{figure}[b]
    \includegraphics[scale=0.37]{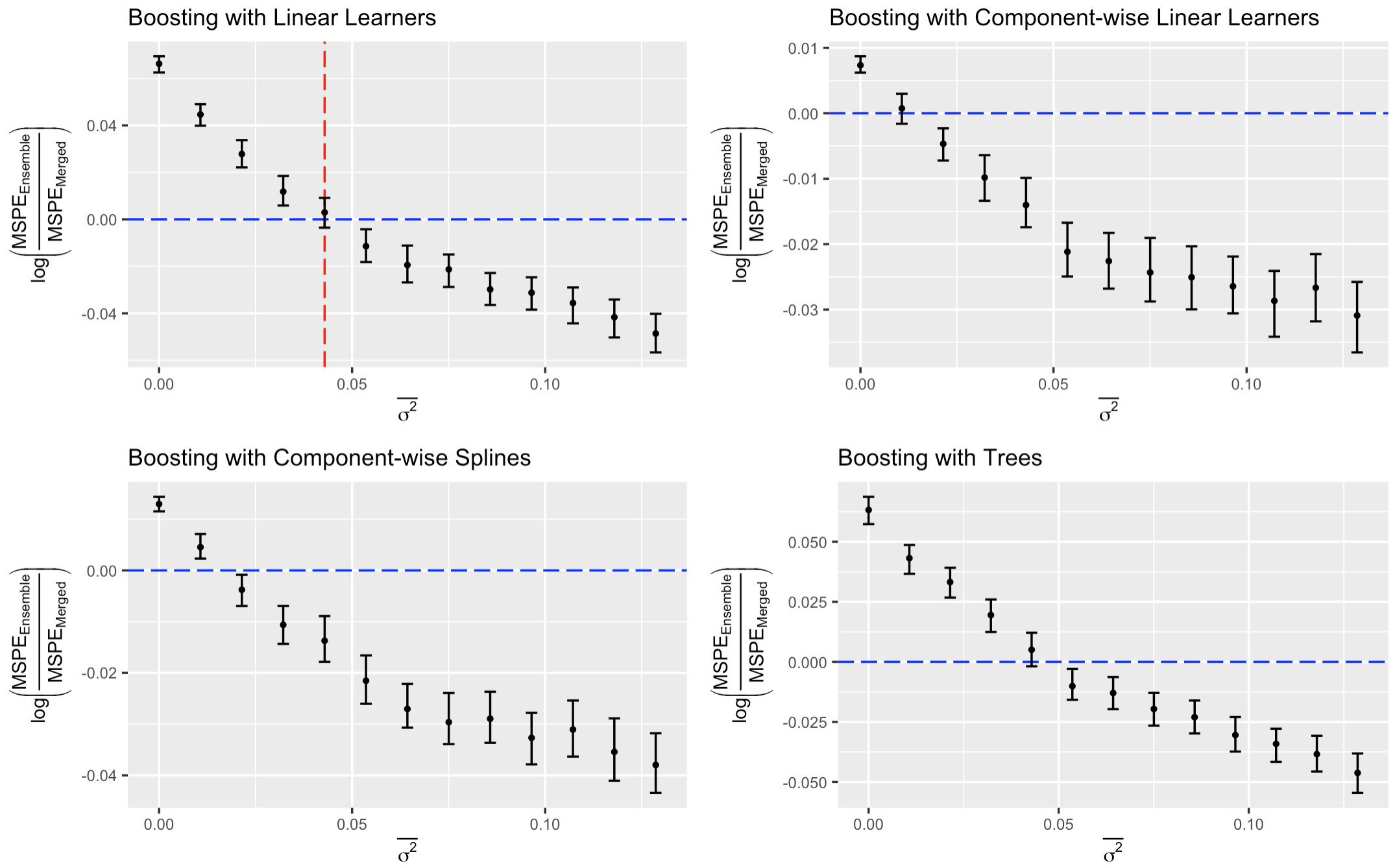}
    \caption{Log relative mean squared prediction error (MSPE) of multi-study ensembling vs. merging for boosting with different base learners under the equal variance assumption. The red vertical dashed line indicates the transition point $\tau$. The solid circles represent the average performance ratios comparing multi-study ensembling to merging, and vertical bars the 95\% bootstrapped intervals.}
    \label{fig:1}
\end{figure}
\begin{figure}[h!]
    \includegraphics[scale=0.37]{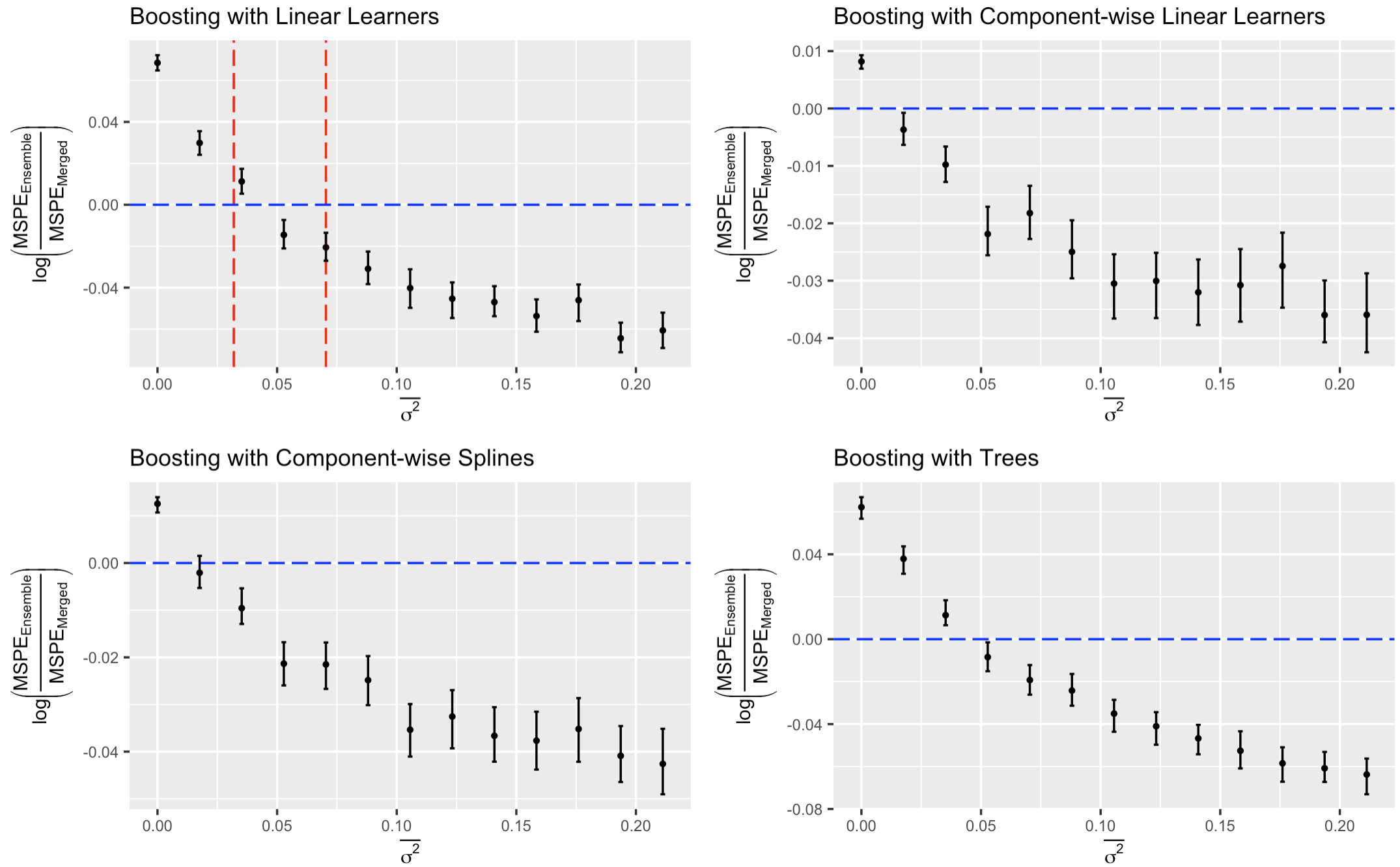}
     \caption{Log relative mean squared prediction error (MSPE) of multi-study ensembling vs. merging for boosting with different base learners under the unequal variance assumption. The red vertical dashed lines indicates the transition interval $[\tau_1, \tau_2]$. The solid circles represent the average performance ratios comparing multi-study ensembling to merging, and vertical bars the 95\% bootstrapped intervals.}
    \label{fig:2}
\end{figure}
\clearpage
\begin{figure}[h!]
    \centering
    \includegraphics[scale=0.35]{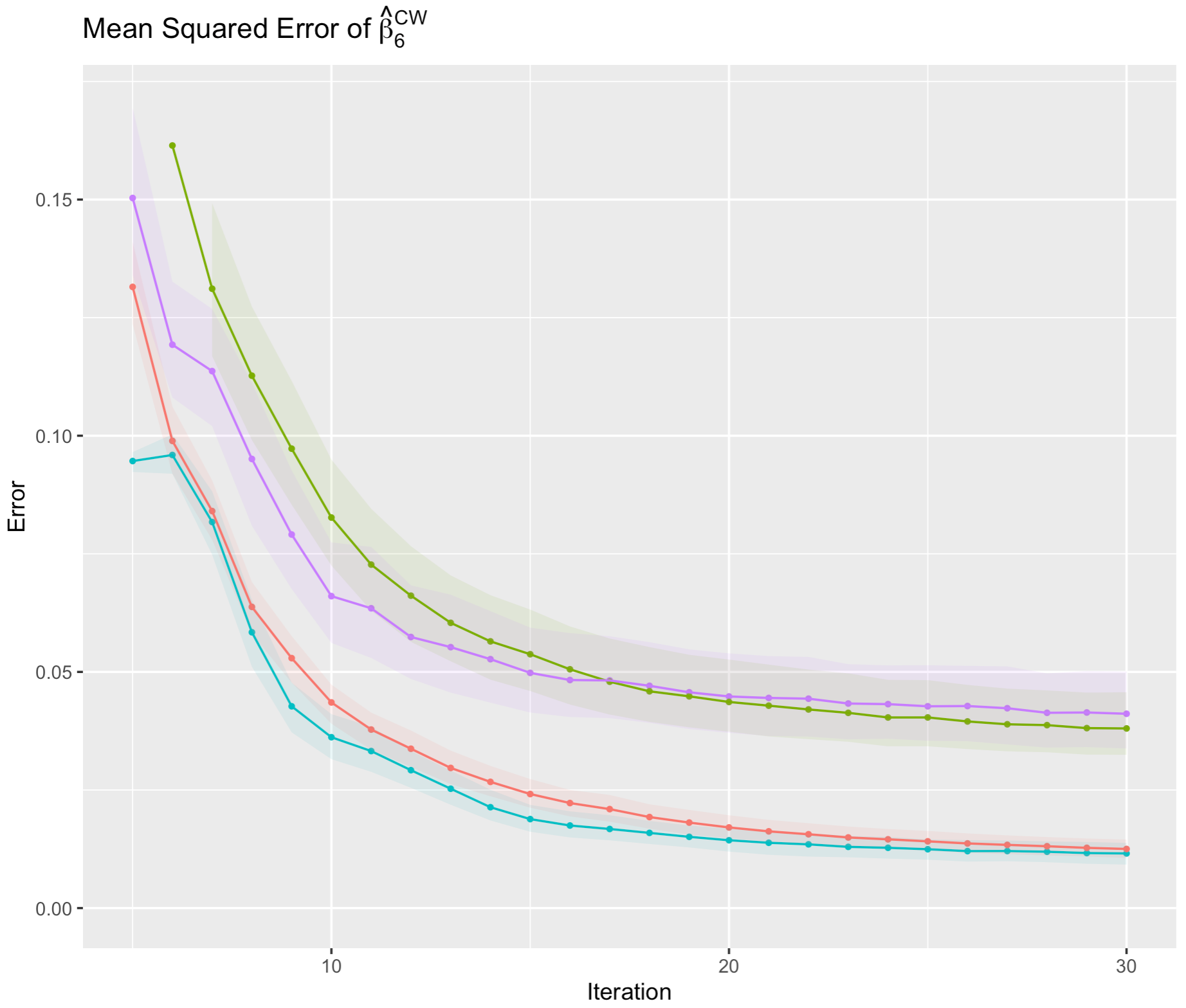}
    \caption{Mean squared error associated with merging and ensembling at different levels of $\overline{\sigma}^2$. Blue and red lines correspond to the merged and ensemble estimators at $\overline{\sigma}^2 = 0.01$, respectively. Purple and green lines correspond to the merged and ensemble estimators at $\overline{\sigma}^2 = 0.05$, respectively.}
    \label{fig:3}
\end{figure}

\begin{figure}[h!]
    \includegraphics[scale=0.37]{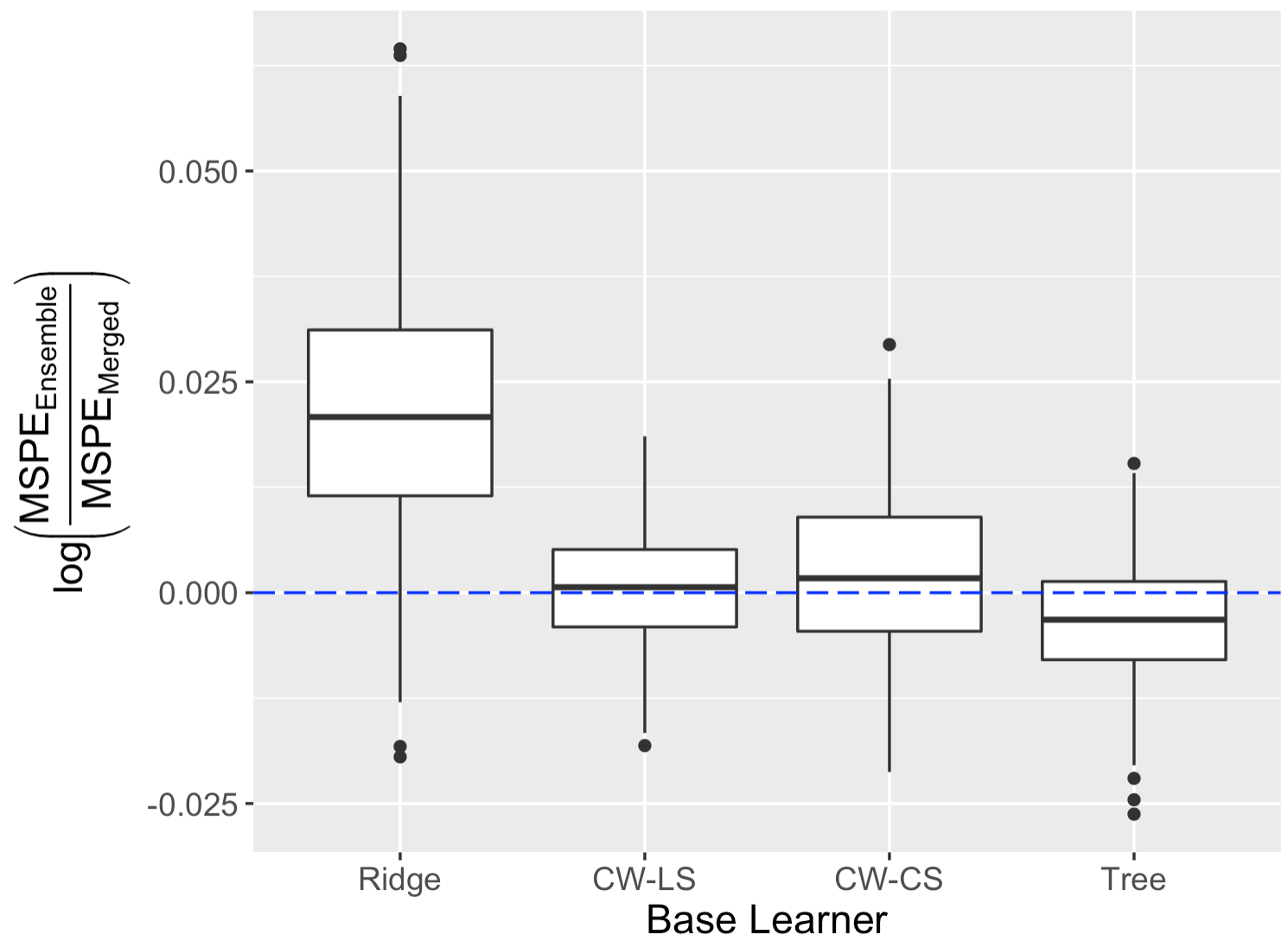}
    \caption{Log relative mean squared prediction error (MSPE) of multi-study ensembling vs. merging for boosting with different base learners under the equal variance assumption. Ridge = ridge regression; CW-LS = component-wise least squares; CW-CS = component-wise cubic smoothing splines; tree = regression tree.}
    \label{fig:4}
\end{figure}

\begin{table}[ht]
\label{tab:1}
\resizebox{\columnwidth}{!}{
\begin{tabular}{cllllllll}
  \hline
Learner & ID 1379 & ID 2034 & ID 9893 & ID 19615 & ID 21974 & Merged \\ 
  \hline
CW-LS & S100P (0.135) & MMP11 (0.0455) & PPP1R3C (0.0421) & CENPN (0.111) & CD69 (0.193) & S100P (0.0215) \\ 
   & AEBP1 (0.129) & CENPA (0.0241) & IGF1 (0.0208) & CD9 (0.0767) & MMP11 (0.108) & ASPN (0.0184) \\ 
   &CENPA (0.0652) & CAMP (0.0204) & SYT1 (0.0183) & ASPN (0.0733) & ESR1 (0.0358) & SYT1 (0.0133) \\ 
   \hline
CW-CS & AEBP1 (0.133) & TNFSF4 (0.0477) & PPP1R3C (0.0463) & CENPN (0.103) & MMP11 (0.183) & S100P (0.021) \\ 
   & C10orf116 (0.115) & S100A9 (0.0405) & GRP (0.0342) & CD9 (0.0865) & CD69 (0.182) & MMP11 (0.0195) \\ 
   & S100P (0.100) & CLU (0.0321) & POSTN (0.0256) & COL1A1 (0.0848) & S100P(0.0889) & E2F8 (0.0185) \\ 
    \hline
Tree & S100P (0.111) & S100A9 (0.0699) & PPP1R3C (0.0438) & COL1A1 (0.131) & CD69 (0.147) & MMP11 (0.0286) \\ 
   & N/A & MMP11 (0.0588) & N/A & CD9 (0.108) & N/A & S100P (0.0266) \\ 
   & N/A & N/A & N/A & ADRA2A (0.0732) & N/A & E2F8 (0.0249) \\ 
   \hline
\end{tabular}
}
\caption{Selected genes ordered by decreasing variable importance across different training studies. Each entry in the table consists of the gene name followed by the amount of reduction in training error that is attributed to selecting the gene in parentheses. An entry is N/A if there were fewer than three selected genes. CW-LS = component-wise least squares and CW-CS = component-wise cubic smoothing splines.}
\end{table}
\end{section}
\clearpage
\begin{section}{Appendix}
\begin{proof}[Proof of Proposition \ref{prop:1}]
We show $r_{(m)} = \prod_{\ell=0}^{m-1}\left(I -  \eta H_{(m-\ell-1)}\right) Y$ by induction. Without loss of generality, we assume $\eta = 1$. At iteration 1, the residual vector is
\begin{align*}
     r_{(1)} &=  Y - \hat{ Y}_{(0)}\\
    &= \left( I -  H_{(0)}\right) Y
\end{align*}

At iteration $m-1$, we assume the induction hypothesis:
\begin{align}
     r_{(m-1)} &=  \prod_{\ell=0}^{m-2}\left( I -  H_{(m-\ell-1)}\right) Y \label{eqstar}
\end{align}
At iteration $m,$ the residual vector is
\begin{align*}
     r_{(m)} &=  Y - \hat{ Y}_{(m-1)}\\
    &=  Y - \left(\hat{ Y}_{(m-2)} +  H_{(m-1)} r_{(m-1)}\right)\\
    &=  r_{(m-1)} -  H_{(m-1)} r_{(m-1)}\\
    &= \left( I -  H_{(m-1)}\right) r_{(m-1)}\\
    &\overset{(\ref{eqstar})}{=} \left( I -  H_{(m-1)}\right)\left( I -  H_{(m-2)}\right)\cdots \left( I -  H_{(1)}\right)\left( I -  H_{(0)}\right)  Y\\
    &= \prod_{\ell=0}^{m-1}\left( I -  H_{(m-\ell-1)}\right) Y.
\end{align*}
It follows that $\left(\tilde{X}_{\hat{j}_{(m)}}^T \tilde{X}_{j_{m}} \right)^{-1}\tilde{X}^T_{\hat{j}_{(m)}}r_{(m)} \in \mathbb{R}$ is the coefficient estimate of $\tilde{X}_{\hat{j}_{(m)}}$. Multiplying the coefficient estimate by $e_{\hat{j}_{(m)}} \in \mathbb{R}^P$ results in an $U$-dimensional vector with $\left(\tilde{X}_{\hat{j}_{(m)}}^T \tilde{X}_{j_{m}} \right)^{-1}\tilde{X}^T_{\hat{j}_{(m)}}r_{(m)}$ in the $\hat{j}_{(m)}$-th position and 0 everywhere else. The final coefficient estimates are given by the sum across iteration-specific vectors $e_{\hat{j}_{(m)}}\left(\tilde{X}_{\hat{j}_{(m)}}^T \tilde{X}_{j_{m}}\right)^{-1}\tilde{X}^T_{\hat{j}_{(m)}}r_{(m)}$ for $m = 1, \ldots, M.$ 
\end{proof}
\clearpage
\begin{proof}[Proof of Lemma \ref{lemma:lee}]
We decompose $Y$ into

$$Y = c_j(v_j^TY) + z_j$$

and rewrite the polyhdron as

\begin{align*}
    \{\Gamma Y \geq  0\} &= \left\{ \Gamma \left( c_j  v_j^T Y +  z_j\right) \geq  0\right\}\\
     &= \left\{ \Gamma  c_j  (v_j^T  y) \geq  0 -  \Gamma  z_j\right\}\\
    & = \left\{\left( \Gamma  c_j\right)_{\ell} \left( v_j^T Y\right) \geq 0- ( \Gamma  z_j)_{\ell} \quad \text{for all }\ell = 1, \ldots, 2M(P-1) \right\}\\
    &= \begin{Bmatrix}
 v_j^T Y \geq \frac{ 0 - ( \Gamma  z_j)_{\ell}}{(\Gamma c_j)_{\ell}}, & \text{for } \ell:( \Gamma  c_j)_{\ell} > 0 \\ 
 v_j^T Y \leq \frac{0 - ( \Gamma  z_j)_{ \ell}}{( \Gamma  c)_i}, & \text{for }  \ell:( \Gamma  c_j)_{ \ell} < 0 \\ 
 0 \geq  0 - ( \Gamma  z_j)_{ \ell} & \text{for } \ell:( \Gamma  c_j)_i = 0 
\end{Bmatrix}\\
&= \begin{Bmatrix}
 v_j^T Y \geq \max\limits_{ \ell:( \Gamma  c)_{ \ell} > 0} \frac{ 0 - ( \Gamma  z_j)_{ \ell}}{( \Gamma  c)_i}\\
v_j^T Y \leq \min\limits_{ \ell:( \Gamma  c_j)_{ \ell} < 0}\frac{0 - ( \Gamma  z_j)_{ \ell}}{( \Gamma  c_j)_{ \ell}}\\
 0 \geq \max\limits_{\ell:( \Gamma  c)_{ \ell} = 0}  0 - ( \Gamma  z_j)_{ \ell}
\end{Bmatrix}
\end{align*}
where in the last step, we have divided the components into three categories depending
on whether $(\Gamma c_j)_{\ell} \lesseqgtr 0$, since this affects the direction of the inequality (or
whether we can divide at all). Since $v_j^TY$ is the same quantity for all $\ell$, it must be
at least the maximum of the lower bounds, which is $a_j$, and no more than the
minimum of the upper bounds, which is $b_j.$ Since $a_j, b_j,$ and $c_j$ are independent of $v_j^TY,$ then $v_j^TY$ is conditionally a normal random variable, truncated to be between $a_j$ and $b_j.$ By conditioning on the value of $z_j,$ $$v_j^TY|\{\Gamma Y \geq 0, z_j = z\}$$
is a truncated normal.

\end{proof}
\clearpage
\begin{proof}[Proof of Theorems \ref{thm:1} and \ref{thm:2}]
\begin{flalign*}
    Bias\left(\tilde{X}_0\hat{ \beta}_{(M)}^{\text{Merge}}\right) &= E\left(\tilde{X}_0\sum_{m=1}^M  \eta B \left( I - \eta H\right)^{m-1} Y\right) -  f(\tilde{X}_0)\\
    &= \tilde{X}_0\tilde{R}f(\tilde{X}) -  f(\tilde{X}_0)\\
    Bias\left(\tilde{X}_0\hat{ \beta}_{(M)}^{\text{Ens}}\right) &= E\left(\tilde{X}_0\sum_{k=1}^K w_k \left[\sum_{m=1}^M  \eta B_{k} \left( I - \eta H_{k}\right)^{m-1} Y_k \right]\right) -  f(\tilde{X}_0)&\\
    &= \sum_{k=1}^K w_k \tilde{X}_0 \tilde{R}_k  f(X_k) -  f(\tilde{X}_0)\\
    Cov\left(\tilde{X}_0\hat{ \beta}^{\text{Merge}}_{(M)}\right) &= Cov\left(\tilde{X}_0\sum_{m=1}^M  \eta B \left( I - \eta H\right)^{m-1} Y\right)\\
    &=  \tilde{X}_0\tilde{R} Cov( Y)  \tilde{R}^T \tilde{X}^T_0\\
    &=  \tilde{X}_0\tilde{R} \text{blkdiag}\left(\left\{Cov\left( Y_k\right)\right\}_{k=1}^K\right) \tilde{R}^T\tilde{X}^T_0\\
     &=  \tilde{X}_0\tilde{R} \text{blkdiag}\left(\left\{ Z_k  G  Z_k^T + \sigma^2_{\epsilon} I\right\}_{k=1}^K\right)  \tilde{R}^T \tilde{X}^T_0\\
    Cov\left(\tilde{X}_0\hat{ \beta}^{\text{Ens}}_{(M)}\right) &= Cov\left(\tilde{X}_0\sum_{k=1}^K  w_k \left[\sum_{m=1}^M  \eta B_{k} \left( I - \eta H_{k}\right)^{m-1} Y_k \right]\right)\\
    &= Cov\left(\tilde{X}_0\sum_{k=1}^K w_k  \tilde{R}_k  Y_k\right)&\\
    &= \sum_{k=1}^K w_k^2  \tilde{X}_0\tilde{R}_k\left(  Z_k  G  Z_k^T + \sigma^2_{\epsilon} I\right)  \tilde{R}_k^T\tilde{X}^T_0\\
    &= \sum_{k=1}^K w_k^2 \tilde{X}_0 \tilde{R}_k Z_k  G  Z_k^T  \tilde{R}_k^T\tilde{X}^T_0 + \sigma^2_{\epsilon}\sum_{k=1}^Kw_k^2  \tilde{X}_0\tilde{R}_k  \tilde{R}_k^T\tilde{X}^T_0\\
\end{flalign*}
Let $b^{\text{Merge}} = Bias\left(\tilde{X}_0 \hat{ \beta}_{(M)}^{\text{Merge}}\right)$. The MSPE of $\hat{\beta}^{\text{Merge}}_{(M)}$ is
{\scriptsize
\begin{align*}
    E\left[\norm{Y_0 - \tilde{X}_0\hat{ \beta}^{\text{Merge}}_{(M)}}^2_2\right] &= \text{tr}\left(Cov\left(\tilde{X}_0\hat{ \beta}_{(M)}^{\text{Merge}}\right)\right) + \left( b^{\text{Merge}}\right)^T b^{\text{Merge}} + E\left[\norm{Y_0 - f(\tilde{X}_0)}^2_2\right]\\
    &= \text{tr}\left(\tilde{X}_0\tilde{R} \text{blkdiag}\left(\left\{Cov( Y_k)\right\}_{k=1}^K\right)  \tilde{R}^T\tilde{X}^T_0\right) + \left( b^{\text{Merge}}\right)^T b^{\text{Merge}} + E\left[\norm{Y_0 - f(\tilde{X}_0)}^2_2\right]\\
    &= \text{tr}\left( \text{blkdiag}\left(\left\{ Z_k  G  Z_k^T + \sigma^2_{\epsilon} I\right\}_{k=1}^K\right)  \tilde{R}^T\tilde{X}^T_0\tilde{X}_0  \tilde{R}\right) + \left( b^{\text{Merge}}\right)^T b^{\text{Merge}}+ E\left[\norm{Y_0 - f(\tilde{X}_0)}^2_2\right]\\
    &= \text{tr}\left(\text{blkdiag}\left(\{ Z_k  G  Z_k^T\}_{k=1}^K\right)\tilde{R}^T\tilde{X}^T_0\tilde{X}_0  \tilde{R}\right) + \sigma^2_{\epsilon}\text{tr}\left(\tilde{R}^T\tilde{X}^T_0\tilde{X}_0  \tilde{R}\right)+  \left( b^{\text{Merge}}\right)^T b^{\text{Merge}}+ E\left[\norm{Y_0 - f(\tilde{X}_0)}^2_2\right]\\
    &= \text{tr}\left( Z'  G'  Z'^T  \tilde{R}^T\tilde{X}^T_0\tilde{X}_0  \tilde{R}\right) + \sigma^2_{\epsilon}\text{tr}\left( \tilde{R}^T\tilde{X}^T_0\tilde{X}_0  \tilde{R}\right)+  \left( b^{\text{Merge}}\right)^T b^{\text{Merge}}+ E\left[\norm{Y_0 - f(\tilde{X}_0)}^2_2\right]\\
    &= \text{tr}\left( G' Z'^T  \tilde{R}^T\tilde{X}^T_0\tilde{X}_0  \tilde{R} Z'\right) + \sigma^2_{\epsilon}\text{tr}\left( \tilde{R}^T\tilde{X}^T_0\tilde{X}_0  \tilde{R}\right)+  \left( b^{\text{Merge}}\right)^T b^{\text{Merge}}+ E\left[\norm{Y_0 - f(\tilde{X}_0)}^2_2\right]\\
   &= \sum_{d=1}^D \sigma_{(d)}^2  \left\{\sum_{i: \sigma^2_{i} = \sigma^2_{(d)}} \left[\sum_{k=1}^K \left( Z'^T  \tilde{R}^T\tilde{X}^T_0\tilde{X}_0  R Z'\right)_{i + Q \times (k - 1), i + Q \times (k - 1)}\right] \right\}
   + \sigma^2_{\epsilon}\text{tr}\left( \tilde{R}^T\tilde{X}^T_0\tilde{X}_0  \tilde{R}\right) \\
   &+  \left( b^{\text{Merge}}\right)^T b^{\text{Merge}}+ E\left[\norm{Y_0 - f(\tilde{X}_0)}^2_2\right]
\end{align*}
}

Let $ b^{\text{Ens}} = Bias\left(\tilde{X}_0 \hat{\beta}_{(M_{\text{Ens}})}^{\text{Ens}}\right)$. The MSPE of $\hat{ \beta}^{Ens}_{(M_{\text{Ens}})}$ is
{\scriptsize
\begin{align*}
    E\left[\norm{ Y_0 - \tilde{X}_0\hat{ \beta}^{\text{Ens}}_{(M_{\text{Ens}})}}^2_2\right] &= \text{tr}\left(Cov\left(\tilde{X}_0\hat{ \beta}_{(M_{\text{Ens}})}^{\text{Ens}}\right)\right) + \left( b^{\text{Ens}}\right)^T b^{\text{Ens}}+ E\left[\norm{Y_0 - f(\tilde{X}_0)}^2_2\right]\\
    &= \text{tr}\left(\tilde{X}_0Cov\left(\sum_{k=1}^K w_k  \tilde{R}_k  Y_k\right)\tilde{X}_0^T\right) + \left( b^{\text{Ens}}\right)^T b^{\text{Ens}}+ E\left[\norm{Y_0 - f(\tilde{X}_0)}^2_2\right]\\
    &= \sum_{k=1}^K w_k^2 \text{tr}\left( Z_k  G  Z_k^T  \tilde{R}_k^T\tilde{X}_0^T\tilde{X}_0 \tilde{R}_k\right) +\sigma^2_{\epsilon} \sum_{k=1}^Kw_k^2 \text{\text{tr}}\left( \tilde{R}_k^T \tilde{X}_0^T\tilde{X}_0 \tilde{R}_k\right) + \left( b^{\text{Ens}}\right)^T b^{\text{Ens}}+ E\left[\norm{Y_0 - f(\tilde{X}_0)}^2_2\right]\\
    &= \sum_{k=1}^K w_k^2 \text{tr}\left( G  Z_k^T  \tilde{R}_k^T \tilde{X}_0^T\tilde{X}_0\tilde{R}_k Z_k\right) +\sigma^2_{\epsilon} \sum_{k=1}^Kw_k^2 \text{\text{tr}}\left( \tilde{R}_k^T \tilde{X}_0^T\tilde{X}_0  \tilde{R}_k\right) + \left( b^{\text{Ens}}\right)^T b^{\text{Ens}}+ E\left[\norm{Y_0 - f(\tilde{X}_0)}^2_2\right]\\
     &=\sum_{d=1}^D \sigma^2_{(d)} \left\{ \sum_{i:\sigma^2_i = \sigma^2_{(d)}}\left[\sum_{k=1}^K w_k^2 \left(  Z_k^T  \tilde{R}_k^T \tilde{X}_0^T\tilde{X}_0 \tilde{R}_k Z_k\right)_{i,i}\right]\right\}\\
     &+\sigma^2_{\epsilon} \sum_{k=1}^Kw_k^2 \text{\text{tr}}\left( \tilde{R}_k^T \tilde{X}_0^T\tilde{X}_0\tilde{R}_k\right) + \left( b^{\text{Ens}}\right)^T b^{\text{Ens}}+ E\left[\norm{Y_0 - f(\tilde{X}_0)}^2_2\right]
\end{align*}
}
If $\sigma^2_1 =\sigma^2_2 = \ldots = \sigma^2_J$ (Theorem 1), then

{\scriptsize
\begin{align*}
\overline{\sigma}^2 & \geq \frac{Q}{P} \times  \frac{\sigma^2_{\epsilon}\left(\sum_{k=1}^Kw_k^2 \text{\text{tr}}\left( \tilde{R}_k^T  \tilde{X}_0^T \tilde{X}_0  \tilde{R}_k\right)-\text{tr}\left( \tilde{R}^T \tilde{X}_0^T \tilde{X}_0  \tilde{R}\right)\right)+ \left( b^{\text{Ens}}\right)^T b^{\text{Ens}} -   \left( b^{\text{Merge}}\right)^T b^{\text{Merge}}}{\text{tr}\left( Z'^T  \tilde{R}^T \tilde{X}_0^T \tilde{X}_0  \tilde{R}  Z'\right) - \sum_{k=1}^K w_k^2 \text{tr}\left( Z_k^T  \tilde{R}_k^T \tilde{X}_0^T \tilde{X}_0  \tilde{R}_k  Z_k\right)}\\
&\Rightarrow \sigma^2 \left(\text{tr}\left( Z'^T  \tilde{R}^T \tilde{X}_0^T \tilde{X}_0  \tilde{R} Z'\right) - \sum_{k=1}^K w_k^2 \text{tr}\left( Z_k  \tilde{R}_k^T \tilde{X}_0^T \tilde{X}_0  \tilde{R}_k  Z_k\right)\right) \\
& \geq \sigma^2_{\epsilon}\left(\sum_{k=1}^Kw_k^2 \text{\text{tr}}\left( \tilde{R}_k^T  \tilde{X}_0^T \tilde{X}_0  \tilde{R}_k\right)-\text{tr}\left(\tilde{R}^T \tilde{X}_0^T \tilde{X}_0  \tilde{R}\right)\right) + \left( b^{\text{Ens}}\right)^T b^{\text{Ens}} -   \left( b^{\text{Merge}}\right)^T b^{\text{Merge}} \\
    &\Leftrightarrow  E\left[\norm{ Y_0 -  \tilde{X}_0 \hat{ \beta}^{\text{Merge}}_{(M)}}^2_2\right] \geq  E\left[\norm{ Y_0 -  \tilde{X}_0 \hat{ \beta}^{\text{Ens}}_{(M_{\text{Ens}})}}^2_2\right].
\end{align*}
}

If $\sigma^2_j \neq\sigma^2_{j'}$ for at least one $j \neq j'$ (Theorem 2), then let 

$$a_d = \sum_{i: \sigma^2_i = \sigma^2_{(d)}} \left[ \sum_{k=1}^K  \left(Z'^T\tilde{R}^T\tilde{X}_0^T\tilde{X}_0 \tilde{R}Z'\right)_{i + Q \times (k - 1), i + Q \times (k - 1)} -  w_k^2\left(Z_k^T\tilde{R}_k^T\tilde{X}^T_0\tilde{X}_0\tilde{R}_k Z_k\right)_{i,i}\right]$$
and $$c = \sigma^2_{\epsilon}\left(\sum_{k=1}^Kw_k^2 \text{\text{tr}}\left( \tilde{R}_k^T  \tilde{X}_0^T \tilde{X}_0  \tilde{R}_k\right)-\text{tr}\left( \tilde{R}^T \tilde{X}_0^T \tilde{X}_0  \tilde{R}\right)\right)+\left( b^{\text{Ens}}\right)^T b^{\text{Ens}} - \left( b^{\text{Merge}}\right)^T b^{\text{Merge}}.$$
Since 
$$E\left[\norm{ Y_0 -  \tilde{X}_0 \hat{ \beta}^{\text{Merge}}_{(M)}}^2_2\right] \geq  E\left[\norm{ Y_0 -  \tilde{X}_0 \hat{ \beta}^{\text{Ens}}_{(M_{\text{Ens}})}}^2_2\right] \Longleftrightarrow \sum_{d=1}^D \sigma^2_{(d)} a_d \geq c$$
and $$\left(\min_d \frac{a_d}{J_d}\right) \sum_{d=1}^D \sigma^2_{(d)} J_d \leq \sum_{d=1}^D \sigma^2_{(d)} \leq \left(\max_d \frac{a_d}{J_d}\right) \sum_{d=1}^D \sigma^2_{(d)} J_d,$$
assuming $a_d > 0$ for all $d$, then 
\begin{align*}
    \overline{\sigma}^2 &= \frac{\sum_{d=1}^D \sigma^2_{(d)} J_d}{P} \leq \frac{c}{P \max_{d} \frac{a_d}{J_d}} = \tau_1 \\
    &\Rightarrow \sum_{d=1}^D \sigma^2_{(d)} a_d \leq \max_{d} \frac{a_d}{J_d} \sum_{d=1}^D \sigma^2_{(d)} J_d \leq c\\
    &\Longleftrightarrow E\left[\norm{ Y_0 -  \tilde{X}_0 \hat{ \beta}^{\text{Merge}}_{(M)}}^2_2\right] \leq  E\left[\norm{ Y_0 -  \tilde{X}_0 \hat{ \beta}^{\text{Ens}}_{(M_{\text{Ens}})}}^2_2\right].
\end{align*}
and 
\begin{align*}
    \overline{\sigma}^2 &= \frac{\sum_{d=1}^D \sigma^2_{(d)} J_d}{P} \geq \frac{c}{P \max_{d} \frac{a_d}{J_d}} = \tau_2 \\
    &\Rightarrow \sum_{d=1}^D \sigma^2_{(d)} a_d \geq \min_{d} \frac{a_d}{J_d} \sum_{d=1}^D \sigma^2_{(d)} J_d \geq c\\
    &\Longleftrightarrow E\left[\norm{ Y_0 -  \tilde{X}_0 \hat{ \beta}^{\text{Merge}}_{(M)}}^2_2\right] \geq  E\left[\norm{ Y_0 -  \tilde{X}_0 \hat{ \beta}^{\text{Ens}}_{(M_{\text{Ens}})}}^2_2\right].
\end{align*}

\end{proof}

\clearpage

\begin{proof}[Proof of Proposition \ref{prop:2}]
\begin{align*}
    Var\left(\left.\hat{\beta}^{\text{Merge, CW}}_{(M)j}\right|\mathcal{P}\right) &= \vartheta^2_j \left(1 - \frac{\xi_j\phi(\xi_j) - \alpha_j\phi(\alpha_j)}{\Phi(\xi_j) - \Phi(\alpha_j)} - \left(\frac{\phi(\xi_j) - \phi(\alpha_j)}{\Phi(\xi_j) - \Phi(\alpha_j)}\right)^2\right)\\
    Bias^2\left(\left.\hat{\beta}^{\text{Merge, CW}}_{(M)j}\right|\mathcal{P}\right) &= \left( \bar{\mu}_j - \vartheta_j\left(\frac{\phi(\xi_j)-\phi(\alpha_j)}{\Phi(\xi_j) - \Phi(\alpha_j)}\right) - \beta_j\right)^2\\
    Var\left(\left.\hat{\beta}^{\text{Ens, CW}}_{(M)j}\right|\mathcal{P}^{\text{Ens}}\right) &= Var\left(\left.\sum_{k=1}^K w_k \hat{\beta}^{\text{CW}}_{(M_k)jk}\right|\mathcal{P}^{\text{Ens}}\right)\\
    &= \sum_{k=1}^K w_k^2Var\left(\left. \hat{\beta}^{\text{CW}}_{(M_k)jk}\right|\mathcal{P}_k\right)  
    \\
    &= \sum_{k=1}^K w_k^2 \vartheta^2_{jk} \left(1 - \frac{\xi_{jk}\phi(\xi_{jk}) - \alpha_{jk}\phi(\alpha_{jk})}{\Phi(\xi_{jk}) - \Phi(\alpha_{jk})} - \left(\frac{\phi(\xi_{jk}) - \phi(\alpha_{jk})}{\Phi(\xi_{jk}) - \Phi(\alpha_{jk})}\right)^2\right)\\
    Bias^2\left(\left.\hat{\beta}^{\text{Ens, CW}}_{(M)j}\right|\mathcal{P}^{\text{Ens}}\right) &= \left(\sum_{k=1}^K w_k E\left(\left.\hat{\beta}^{\text{CW}}_{(M_k)jk}\right|\mathcal{P}^\text{Ens}\right) - \beta_j\right)^2\\
    &= \left(\sum_{k=1}^K w_k E\left(\left.\hat{\beta}^{\text{CW}}_{(M_k)jk}\right|\mathcal{P}_k\right) - \beta_j\right)^2 
    \\
    &= \left(\sum_{k=1}^K w_k  \left(\bar{\mu}_{jk} - \vartheta_{jk}\left(\frac{\phi(\xi_{jk})-\phi(\alpha_{jk})}{\Phi(\xi_{jk}) - \Phi(\alpha_{jk})}\right)\right) - \beta_j\right)^2
\end{align*}
\end{proof}
\newpage
\begin{claim}[Truncation region for component-wise boosting coefficients]
\label{claim:1}
Let $Y \in \mathbb{R}^N$ denote the outcome vector where $Y \sim N(\mu, \Sigma)$. The boosting coefficients can be written as
\begin{align*}
    \hat{\beta}^{\text{CW, Merge}}_{(M)} &= V^TY\\
    &\coloneqq \sum_{m=1}^M \eta B_{(m)}\left(\prod_{\ell = 0}^{m-1} (I - \eta H_{(m - \ell - 1))}\right)Y,
\end{align*}
where $V \in \mathbb{R}^{N \times P}$ depends on $Y$ through variable selection. We decompose $ Y$ into
$$ Y =  C( V^T Y) +  Z^*,$$
where $$ C =  \Sigma  V\left( V^T \Sigma  V\right)^{-1},$$
is a $N-$dimensional vector and 
$$ Z^* = \left( I -  \Sigma  V\left( V^T \Sigma  V\right)^{-1}  V^T\right) Y$$
is a $\ell_P \coloneqq 2M(P-1)$ dimensional vector. We claim the polyhedral set $\{ \Gamma  Y \geq  0\}$ can be re-written as a truncation region where the coefficients $\hat{\beta}^{\text{CW, Merge}}_{(M)}$ have non-rectangular truncation limits.
\end{claim}
\begin{proof}
We define the projection $\Pi_k(S)$ of a set $S \subset \mathbb{R}^n$ by letting
$$\Pi_k(S) = \left\{(x_1, \ldots, x_k)| \exists x_{k+1}, \ldots, x_n \text{ s.t. } (x_1, \ldots, x_n) \in S\right\}.$$

Given a polyhedron $\mathcal{P}$ in terms of linear inequality constraints of the form
$$ A  x \geq  b,$$
we state the Fourier Motzkin elimination algorithm from \cite{bertsimas1997introduction}.
\begin{algorithm}
\caption{Elimination algorithm for a system of linear inequalities}
\begin{algorithmic}[1]

\State Rewrite each constraint $\sum_{j=1}^N a_{ij} x_j \geq b_i$ in the form $$a_{iNx_N} \geq -\sum_{j=1}^{N-1} a_{ij}x_j + b_i, \quad i = 1, \ldots, m$$
if $a_{iN} \neq 0$, divide both sides by $a_{iN}.$ By letting $\bar{x} = (x_1, \ldots, x_{n-1}),$ we obtain an equivalent representation of $\mathcal{P}$ involving the following constraints
\begin{align*}
    x_N \geq d_i + f'_i \bar{x}, \qquad &\text{if } a_{iN} > 0\\
    d_j + f'_j\bar{x} \geq x_N, \qquad &\text{if } a_{jN} < 0\\
    0 \geq d_k + f'_k \bar{x}, \qquad &\text{if } a_{kN} = 0
\end{align*}
Each $d_i, d_j, d_k$ is a scalar, and each $f_i, f_j, f_k$ is a vector in $\mathbb{R}^{N-1}$.
\State Let $\mathcal{Q}$ be the polyhedron in $\mathbb{R}^{N-1}$ defined by the constraints
\begin{align*}
    d_j + f'_j\bar{x} \geq d_i + f'_i \bar{x} \qquad &\text{if } a_{iN} > 0 \text{ and } a_{jN} < 0\\
    0 \geq d_k +f'_k\bar{x}, \qquad &\text{if } a_{kN} = 0 
\end{align*}
\end{algorithmic}
\end{algorithm}

We note the following:
\begin{enumerate}
    \item The projection $\Pi_k(\mathcal{P})$ can be generated by repeated application of the elimination algorithm (Theorem 2.10 in \cite{bertsimas1997introduction})
    \item The elimination approach always produces a polyhedron (definition of the elimination algorithm in \cite{bertsimas1997introduction}).
\end{enumerate}
Therefore, it follows that a projection $\Pi_k(\mathcal{P})$ of a polyhedron is also a polyhedron. 

The polyhedral set $\mathcal{P} \coloneqq \{ Y:  \Gamma  Y \geq  0\}$ is a system of $\ell_P \coloneqq 2M(P-1)$ linear inequalities, with $P$ variables $ V^T Y_1, \ldots,  V^T Y_P.$ Let $( A)_{ij}$ denote the $i,j$-th entry in matrix $ A$. We let $I_P = \{1, 2, \ldots, \ell_P\}$ denote the row index set for the system of inequalities with $P$ variables and partition it into subsets $I_P^+, I_P^-,$ and $I_P^0$, where $I_P^+ = \{i: ( \Gamma  C)_{ip} > 0\}, I_P^- = \{i: ( \Gamma  C)_{ip} < 0\},$ and $I_P^0 = \{i: ( \Gamma  C)_{ip} = 0\}$. Then we have
\begin{small}

    \begin{align*}
     \{ \Gamma  Y \geq  0\} &= \left\{ \Gamma \left( C  V^T Y +  Z^*\right) \geq  0\right\}\\
     &= \left\{ \underbrace{ \Gamma  C}_{\ell_P \times P} \underbrace{ V^T  Y}_{P \times 1} \geq \underbrace{ 0 -  \Gamma  Z^*}_{\ell_P \times 1}\right\}\\
      &= \left\{\sum_{j=1}^P ( \Gamma  C)_{ij}( V^T Y)_j \geq 0 - ( \Gamma  Z^*)_i \quad  i = 1, \ldots, \ell_P\right\}\\
        &= \left\{( \Gamma  C)_{ip}( V^T  Y)_{p} \geq -\sum_{j=1}^{P-1} ( \Gamma  C)_{ij}( V^T Y)_j - ( \Gamma  Z^*)_i \quad i = 1,\ldots, \ell_P\right\}\\
        &= \begin{Bmatrix}
( V^T Y)_P \geq \frac{-\sum_{j=1}^{P-1} ( \Gamma  C)_{qj}( V^T Y)_j - ( \Gamma  Z^*)_q}{( \Gamma  C)_{qp}}, & \text{for } q \in I_P^+\\ 
( V^T Y)_P \leq \frac{-\sum_{j=1}^{P-1} ( \Gamma  C)_{rj}( V^T Y)_j - ( \Gamma  Z^*)_r}{( \Gamma  C)_{rp}}, & \text{for } r \in I_P^-\\ 
0 \geq -\sum_{j=1}^{P-1}( \Gamma  C)_{sj}( V^T Y)_j - ( \Gamma  Z^*)_s & \text{for } s \in I_P^0\\ 
\end{Bmatrix}\\
&= \begin{Bmatrix}\max_{q \in I^+} \frac{-\sum_{j=1}^{P-1} ( \Gamma  C)_{qj}( V^T Y)_j - ( \Gamma  Z^*)_q}{( \Gamma  C)_{qp}} \leq ( V^T Y)_{p} \leq \min_{r \in I^-} \frac{-\sum_{j=1}^{P-1} ( \Gamma  C)_{rj}( V^T Y)_j - ( \Gamma  Z^*)_r}{( \Gamma  C)_{rp}}\\
0 \geq -\sum_{j=1}^{P-1}( \Gamma  C)_{sj}( V^T Y)_j - ( \Gamma  Z^*)_s & \text{for } s \in I_P^0
\end{Bmatrix}
\end{align*}
\end{small}

We reduce this to a system of inequalities with $P-1$ variables after eliminating $( V^T Y)_P$:
\begin{equation}\label{eqn:1}
    \begin{Bmatrix}
\frac{-\sum_{j=1}^{P-1} ( \Gamma  C)_{qj}( V^T Y)_j - ( \Gamma  Z^*)_q}{( \Gamma  C)_{qp}} \leq \frac{-\sum_{j=1}^{P-1} ( \Gamma  C)_{rj}( V^T Y)_j - ( \Gamma  Z^*)_r}{( \Gamma  C)_{rp}} \text{ for } q \in I_P^+, r \in I_P^-\\
 0 \geq -\sum_{j=1}^{P-1}( \Gamma  C)_{sj}( V^T Y)_j - ( \Gamma  Z^*)_s \text{ for } s \in I_P^0
 \end{Bmatrix}
\end{equation}
The set in (\ref{eqn:1}) is a system of $\ell_{P-1} \coloneqq |I_P^+| \times |I_P^-| + |I_P^0|$ inequalities. It is a polyhedral set in $\mathbb{R}^{P-1}$, which can be seen by rewriting (\ref{eqn:1}) as follows:
\newpage
\begin{scriptsize}
\begin{align*}
   & \begin{Bmatrix}
\frac{-\sum_{j=1}^{P-1} ( \Gamma  C)_{qj}( V^T Y)_j - ( \Gamma  Z^*)_q}{( \Gamma  C)_{qp}} \leq \frac{-\sum_{j=1}^{P-1} ( \Gamma  C)_{rj}( V^T Y)_j - ( \Gamma  Z^*)_r}{( \Gamma  C)_{rp}} \text{ for } q \in I^+, r \in I^-\\
 0 \geq -\sum_{j=1}^{P-1}( \Gamma  C)_{sj}( V^T Y)_j - ( \Gamma  Z^*)_s \text{ for } s \in I^0
 \end{Bmatrix} \\
 =&\begin{Bmatrix}
-\sum_{j=1}^{P-1} ( \Gamma   C)_{rp}( \Gamma  C)_{qj}( V^T Y)_j -( \Gamma   C)_{rp} ( \Gamma  Z^*)_q \geq -\sum_{j=1}^{P-1} ( \Gamma   C)_{qp}( \Gamma  C)_{rj}( V^T Y)_j - ( \Gamma   C)_{qp}( \Gamma  Z^*)_r \text{ for } q \in I^+, r \in I^-\\
 \sum_{j=1}^{P-1}( \Gamma  C)_{sj}( V^T Y)_j \geq - ( \Gamma  Z^*)_s \text{ for } s \in I^0
 \end{Bmatrix} \\
  =&\begin{Bmatrix}
\sum_{j=1}^{P-1} \left(( \Gamma   C)_{qp}( \Gamma  C)_{rj}- ( \Gamma   C)_{rp}( \Gamma  C)_{qj}\right)( V^T Y)_j  \geq ( \Gamma   C)_{rp} ( \Gamma  Z^*)_q - ( \Gamma   C)_{qp}( \Gamma  Z^*)_r \text{ for } q \in I^+, r \in I^-\\
 \sum_{j=1}^{P-1}( \Gamma  C)_{sj}( V^T Y)_j \geq - ( \Gamma  Z^*)_s \text{ for } s \in I^0
 \end{Bmatrix}.
\end{align*}
\end{scriptsize}

Let $ A_{p-k}$ denote a $\ell_{p-k} \times (p-k)$ matrix, $( V^T Y)_{1:p-k}$ a vector that contains the first $p-k$ coordinates of $( V^T Y)$, and $ b_{p-k}( Z^*)$ a $\ell_{p-k}$-dimensional vector, where $k \in \{0, \ldots, P-1\}$, and $\ell_{p-k}$ is the number of linear constraints in $\Pi_{p-k}(\mathcal{P})$, which is the projection of $\mathcal{P}$. Note that $ A_P= \Gamma  C$ and $ b_P( Z^*) =  0 -  \Gamma  Z^*.$

We repeat the elimination process $P-1$ times to obtain $\Pi_1(\mathcal{P}):$
\begin{align*}
\left\{ \Gamma  Y \geq  0\right\} &= \{ A_{p}( V^T Y) \geq  b_{p}( Z^*)\}\\
\Pi_{P-1}(\mathcal{P}) &= \{ A_{P-1}( V^T Y)_{1:P-1} \geq  b_{P-1}( Z^*)\}\\
        &\vdots\\
            \Pi_{1}(\mathcal{P}) &= \{ A_{1}( V^T Y)_{1} \geq  b_{1}( Z^*)\}.
\end{align*}
\underline{Induction base case for $\Pi_2(\mathcal{P})$:} Without loss of generality, we assume the variable in $\Pi_1(\mathcal{P})$ is $( V^T Y)_1$. We can obtain its lower and upper truncation limits, $\mathcal{V}_1^{\text{lo}}( Z^*)$ and $\mathcal{V}_1^{\text{up}}( Z^*)$, and $\mathcal{V}_1^{0}( Z^*)$ using the same argument as the one in \cite{lee2016exact}, where
\begin{align*}
    \mathcal{V}_1^{\text{lo}}( Z^*) &= \max_{i:( A_1)_i >0} \frac{( b_1( Z^*))_i}{( A_1)_i}\\
    \mathcal{V}_1^{\text{up}}( Z^*) &= \min_{i:( A_1)_i < 0} \frac{( b_1( Z^*))_i}{( A_1)_i}\\
    \mathcal{V}_1^{0}( Z^*) &= \max_{i:( A_1)_i = 0} ( b_1( Z^*))_i.
\end{align*}
We conclude that $\Pi_1(\mathcal{P}) = \{(\mathcal{V}_1^{\text{lo}}( Z^*) \leq ( V^T Y)_1 \leq \mathcal{V}_1^{\text{up}}( Z^*), \mathcal{V}_1^{0}( Z^*) \leq 0\}.$

By the definition of $\Pi_2(\mathcal{P})$, we have 
\begin{align*}
    \Pi_2(\mathcal{P}) &=\left\{ A_2( V^T Y)_{1:2} \geq  b_2( Z^*)\right\}\\
    &= \begin{Bmatrix}
     A_2 ( V^T Y)_{1:2} \geq  b_2( Z^*)\\
    \mathcal{V}_1^{\text{lo}}( Z^*) \leq ( V^T Y)_1 \leq \mathcal{V}_1^{\text{up}}( Z^*)\\
    \mathcal{V}_1^{0}( Z^*) \leq 0
    \end{Bmatrix}
\end{align*}
because reducing the system from $\Pi_2(\mathcal{P})$ to $\Pi_1(\mathcal{P})$ does not change the range of $( V^T Y)_1$ that satisfy the linear constraints in $\Pi_2(\mathcal{P}).$

We can obtain the lower and upper truncation limits for $( V^T Y)_2$ as a function of $( V^T Y)_1$. 
\begin{align*}
    \Pi_2(\mathcal{P}) &= \begin{Bmatrix}
     A_2 ( V^T Y)_{1:2} \geq  b_2( Z^*)\\
    \end{Bmatrix}\\
    &= \begin{Bmatrix}
     A_2 ( V^T Y)_{1:2} \geq  b_2( Z^*)\\
    \mathcal{V}_1^{\text{lo}}( Z^*) \leq ( V^T Y)_1 \leq \mathcal{V}_1^{\text{up}}( Z^*)\\
    \mathcal{V}_1^{0}( Z^*) \leq 0
    \end{Bmatrix}\\
    &= \begin{Bmatrix}
    \sum_{j=1}^2 ( A_2)_{ij}( V^T  Y)_j \geq ( b_2( Z^*))_i \quad \text{ for } i = 1, \ldots, \ell_{2}\\
     \mathcal{V}_1^{\text{lo}}( Z^*) \leq ( V^T Y)_1 \leq \mathcal{V}_1^{\text{up}}( Z^*)\\
    \mathcal{V}_1^{0}( Z^*) \leq 0
    \end{Bmatrix}\\
    &=\begin{Bmatrix}
     ( A_2)_{i2} ( V^T Y)_2 \geq -( A_2)_{i1}( V^T Y)_1 + ( b_2( Z^*))_i \quad \text{ for } i = 1, \ldots, \ell_{2}\\
     \mathcal{V}_1^{\text{lo}}( Z^*) \leq ( V^T Y)_1 \leq \mathcal{V}_1^{\text{up}}( Z^*)\\
    \mathcal{V}_1^{0}( Z^*) \leq 0
    \end{Bmatrix}\\
    &= \begin{Bmatrix}
    ( V^T Y)_2 \geq \frac{-( A_2)_{i1}( V^T Y)_1 + ( b_2( Z^*))_i}{( A_2)_{i2}} \quad \text{ for } i: ( A_2)_{i2} > 0\\
    ( V^T Y)_2 \leq \frac{-( A_2)_{i1}( V^T Y)_1 + ( b_2( Z^*))_i}{( A_2)_{i2}} \quad \text{ for } i: ( A_2)_{i2} < 0\\
    0 \geq -( A_2)_{i1}( V^T Y)_1 + ( b_2( Z^*))_i \quad \text{ for } i: ( A_2)_{i2} = 0\\
     \mathcal{V}_1^{\text{lo}}( Z^*) \leq ( V^T Y)_1 \leq \mathcal{V}_1^{\text{up}}( Z^*)\\
    \mathcal{V}_1^{0}( Z^*) \leq 0
    \end{Bmatrix}\\
     &= \begin{Bmatrix}
    \max\limits_{i:( A_2)_{i2} > 0} \frac{-( A_2)_{i1}( V^T Y)_1 + ( b_2( Z^*))_i}{( A_2)_{i2}}\leq ( V^T Y)_2 \leq \min\limits_{i:( A_1)_{i2} < 0}\frac{-( A_2)_{i1}( V^T Y)_1 + ( b_2( Z^*))_i}{( A_2)_{i2}}\\
    0 \geq \max\limits_{i:( A_2)_{i2} = 0} -( A_2)_{i1}( V^T Y)_1 + ( b_2( Z^*))_i\\
     \mathcal{V}_1^{\text{lo}}( Z^*) \leq ( V^T Y)_1 \leq \mathcal{V}_1^{\text{up}}( Z^*)\\
    \mathcal{V}_1^{0}( Z^*) \leq 0
    \end{Bmatrix}\\
    &=\begin{Bmatrix}
\mathcal{V}^{\text{lo}}_2( Z^*, ( V^T Y)_1) \leq ( V^T Y)_2 \leq \mathcal{V}^{\text{up}}_2( Z^*, ( V^T Y)_1)\\
\mathcal{V}^0_2( Z^*, ( V^T Y)_1) \leq 0\\
\mathcal{V}_{1}^{\text{lo}}( Z^*) \leq ( V^T  Y)_{1} \leq \mathcal{V}_{1}^{\text{up}}( Z^*)\\
\mathcal{V}_{1}^0( Z^*) \leq 0
\end{Bmatrix}
\end{align*}
where
\begin{align*}
    \mathcal{V}_2^{\text{lo}}( Z^*, ( V^T Y)_1) &= \max_{i: ( A_2)_{i2} > 0} \frac{-( A_2)_{i1}( V^T Y)_1( Z^*)_1 + ( b_2( Z^*))_i}{( A_2)_{i2}}\\
    \mathcal{V}_2^{\text{up}}( Z^*, ( V^T  Y)_1) &= \min_{i: ( A_2)_{i2} < 0} \frac{-( A_2)_{i1}( V^T Y)_1( Z^*)_1 + ( b_2( Z^*))_i}{( A_2)_{i2}}\\
    \mathcal{V}_2^0( Z^*, ( V^T Y)_1) &= \max_{i: ( A_2)_{i2} = 0} -( A_2)_{i1}( V^T Y)_1+( b_2( Z^*))_i.
\end{align*}

\underline{Inductive step for $\Pi_{P-1}(\mathcal{P})$}: Under the induction hypothesis, we assume 
$$\Pi_{P-2}(\mathcal{P}) = \begin{Bmatrix}
       \mathcal{V}_{1}^{\text{lo}}( Z^*) \leq ( V^T  Y)_{1} \leq \mathcal{V}_{1}^{\text{up}}( Z^*)\\
\mathcal{V}_{1}^0( Z^*) \leq 0\\
\mathcal{V}_2^{\text{lo}}(( V^T Y)_1,  Z^*) \leq ( V^T Y)_2 \leq \mathcal{V}_2^{\text{up}}(( V^T Y)_1,  Z^*)\\
\mathcal{V}_{2}^0(( V^T Y)_1,  Z^*) \leq 0\\
\vdots\\
\mathcal{V}_{P-2}^{\text{lo}}(( V^T Y)_{1:P-3},  Z^*) \leq ( V^T Y)_{P-2} \leq \mathcal{V}_{P-2}^{\text{up}}(( V^T Y)_{1:P-3},  Z^* )\\
\mathcal{V}_{P-2}^0(( V^T Y)_{1:P-3},  Z^*) \leq 0
    \end{Bmatrix}$$

Then we have 
\begin{align*}
    \Pi_{P-1}(\mathcal{P}) &= \begin{Bmatrix}
     A_{P-1} ( V^T Y)_{1:P-1} \geq  b_{P-1}( Z^*)\\
    \end{Bmatrix}\\
    &= \begin{Bmatrix}
     A_{P-1} ( V^T Y)_{1:P-1} \geq  b_{P-1}( Z^*)\\
   \mathcal{V}_{1}^{\text{lo}}( Z^*) \leq ( V^T  Y)_{1} \leq \mathcal{V}_{1}^{\text{up}}( Z^*)\\
\mathcal{V}_{1}^0( Z^*) \leq 0\\
\mathcal{V}_2^{\text{lo}}( Z^*, ( V^T Y)_1) \leq ( V^T Y)_2 \leq \mathcal{V}_2^{\text{up}}( Z^*, ( V^T Y)_1)\\
\mathcal{V}_{2}^0( Z^*, ( V^T Y)_1) \leq 0\\
\vdots\\
\mathcal{V}_{P-2}^{\text{lo}}(( V^T Y)_{1:P-3},  Z^*) \leq ( V^T Y)_{P-2} \leq \mathcal{V}_{P-2}^{\text{up}}(( V^T Y)_{1:P-3},  Z^* )\\
\mathcal{V}_{P-2}^0(( V^T Y)_{1:P-3},  Z^*) \leq 0
    \end{Bmatrix}\\
    &=\begin{Bmatrix}
     ( A_{P-1})_{i(P-1)} ( V^T Y)_{P-1} \geq -\sum_{j=1}^{P-2}( A_{P-1})_{ij}( V^T Y)_j+ ( b_{P-1}( Z^*))_i \quad \text{ for } i = 1, \ldots, \ell_{P-1}\\
   \mathcal{V}_{1}^{\text{lo}}( Z^*) \leq ( V^T  Y)_{1} \leq \mathcal{V}_{1}^{\text{up}}( Z^*)\\
\mathcal{V}_{1}^0( Z^*) \leq 0\\
\mathcal{V}_2^{\text{lo}}( Z^*, ( V^T Y)_1) \leq ( V^T Y)_2 \leq \mathcal{V}_2^{\text{up}}( Z^*, ( V^T Y)_1)\\
\mathcal{V}_{2}^0( Z^*, ( V^T Y)_1) \leq 0\\
\vdots\\
\mathcal{V}_{P-2}^{\text{lo}}(( V^T Y)_{1:P-3},  Z^*) \leq ( V^T Y)_{P-2} \leq \mathcal{V}_{P-2}^{\text{up}}(( V^T Y)_{1:P-3},  Z^* )\\
\mathcal{V}_{P-2}^0(( V^T Y)_{1:P-3},  Z^*) \leq 0
    \end{Bmatrix}\\
    &= \begin{Bmatrix}
    ( V^T Y)_{P-1} \geq \frac{-\sum_{j=1}^{P-2}( A_{P-1})_{ij}( V^T Y)_j+ ( b_{P-1}( Z^*))_i}{( A_{P-1})_{i(P-1)}} \quad \text{ for } i: ( A_{P-1})_{i(P-1)} > 0\\
    ( V^T Y)_{P-1} \leq \frac{-\sum_{j=1}^{P-2}( A_{P-1})_{ij}( V^T Y)_j+ ( b_{P-1}( Z^*))_i}{( A_{P-1})_{i(P-1)}} \quad \text{ for } i: ( A_{P-1})_{i(P-1)} < 0\\
    0 \geq -\sum_{j=1}^{P-2}( A_{P-1})_{ij}( V^T Y)_j+ ( b_{P-1}( Z^*))_i \quad \text{ for } i: ( A_{P-1})_{i(P-1)} = 0\\
    \mathcal{V}_{1}^{\text{lo}}( Z^*) \leq ( V^T  Y)_{1} \leq \mathcal{V}_{1}^{\text{up}}( Z^*)\\
\mathcal{V}_{1}^0( Z^*) \leq 0\\
\mathcal{V}_2^{\text{lo}}( Z^*, ( V^T Y)_1) \leq ( V^T Y)_2 \leq \mathcal{V}_2^{\text{up}}( Z^*, ( V^T Y)_1)\\
\mathcal{V}_{2}^0( Z^*, ( V^T Y)_1) \leq 0\\
\vdots\\
\mathcal{V}_{P-2}^{\text{lo}}(( V^T Y)_{1:P-3},  Z^*) \leq ( V^T Y)_{P-2} \leq \mathcal{V}_{P-2}^{\text{up}}(( V^T Y)_{1:P-3},  Z^* )\\
\mathcal{V}_{P-2}^0(( V^T Y)_{1:P-3},  Z^*) \leq 0
    \end{Bmatrix}.\\
    &= \begin{Bmatrix}
       \mathcal{V}_{1}^{\text{lo}}( Z^*) \leq ( V^T  Y)_{1} \leq \mathcal{V}_{1}^{\text{up}}( Z^*)\\
\mathcal{V}_{1}^0( Z^*) \leq 0\\
\mathcal{V}_2^{\text{lo}}(( V^T Y)_1,  Z^*) \leq ( V^T Y)_2 \leq \mathcal{V}_2^{\text{up}}(( V^T Y)_1,  Z^*)\\
\mathcal{V}_{2}^0(( V^T Y)_1,  Z^*) \leq 0\\
\vdots\\
\mathcal{V}_{P-1}^{\text{lo}}(( V^T Y)_{1:P-2},  Z^*) \leq ( V^T Y)_{P-1} \leq \mathcal{V}_{P-1}^{\text{up}}(( V^T Y)_{1:P-2},  Z^* )\\
\mathcal{V}_{P-1}^0(( V^T Y)_{1:P-2},  Z^*) \leq 0
    \end{Bmatrix}
\end{align*}
where 
\begin{align*}
    \mathcal{V}_{P-1}^{\text{lo}}\left(( V^T Y)_{1:P-2},  Z^*\right) &= \max_{i:( A_{P-1})_{i(P-1)} > 0}  \frac{-\sum_{j-1}^{P-2}( A_{P-1})_{ij}( V^T  Y)_j + ( b_{P-1}( Z^*))_i}{( A_{P-1})_{i(P-1)}} \\
  \mathcal{V}_{P-1}^{\text{up}}\left(( V^T Y)_{1:P-2},  Z^*\right) &= \min_{i:( A_{P-1})_{i(P-1)} < 0}  \frac{-\sum_{j-1}^{P-2}( A_{P-1})_{ij}( V^T  Y)_j + ( b_{P-1}( Z^*))_i}{( A_{P-1})_{i(P-1)}} \\
    \mathcal{V}_{P-1}^{\text{0}}\left(( V^T Y)_{1:P-2},  Z^*\right) &= \max_{\substack{\\ i:( A_{P-1})_{i(P-1)} = 0 \\}}  -\sum_{j=1}^{P-2} ( A_{P-1})_{ij}( V^T Y)_j + ( b_{P-1}( Z^*))_i.
\end{align*}

Therefore, we conclude that 
\begin{align*}
    \Pi_P(\mathcal{P}) &= \left\{ \Gamma  Y \geq  0\right\}\\
    &= \begin{Bmatrix}
       \mathcal{V}_{1}^{\text{lo}}( Z^*) \leq ( V^T  Y)_{1} \leq \mathcal{V}_{1}^{\text{up}}( Z^*)\\
\mathcal{V}_{1}^0( Z^*) \leq 0\\
\mathcal{V}_2^{\text{lo}}(( V^T Y)_1,  Z^*) \leq ( V^T Y)_2 \leq \mathcal{V}_2^{\text{up}}(( V^T Y)_1,  Z^*)\\
\mathcal{V}_{2}^0(( V^T Y)_1,  Z^*) \leq 0\\
\vdots\\
\mathcal{V}_{P-1}^{\text{lo}}(( V^T Y)_{1:P-2},  Z^*) \leq ( V^T Y)_{P-1} \leq \mathcal{V}_{P-1}^{\text{up}}(( V^T Y)_{1:P-2},  Z^* )\\
\mathcal{V}_{P-1}^0(( V^T Y)_{1:P-2},  Z^*) \leq 0\\
\mathcal{V}_{p}^{\text{lo}}(( V^T Y)_{1:P-1},  Z^*) \leq ( V^T Y)_{p} \leq \mathcal{V}_{p}^{\text{up}}(( V^T Y)_{1:P-1},  Z^*)\\
\mathcal{V}_{p}^0(( V^T Y)_{1:P-1},  Z^*) \leq 0\\
    \end{Bmatrix}
\end{align*}
where 
\begin{align*}
    \mathcal{V}_P^{\text{lo}}\left(( V^T Y)_{1:P-1},  Z^*\right) &= \max_{i:( A_P)_{ip} > 0}  \frac{-\sum_{j-1}^{P-1}( A_P)_{ij}( V^T  Y)_j + ( b_P( Z^*))_i}{( A_P)_{ip}} \\
    \mathcal{V}_P^{\text{up}}\left(( V^T  Y)_{1:P-1},  Z^*\right) &= \min_{i:( A_P)_{ip} < 0}  \frac{-\sum_{j-1}^{P-1}( A_P)_{ij}( V^T  Y)_j + ( b_P( Z^*))_i}{( A_P)_{ip}}\\
    \mathcal{V}_P^{\text{0}}\left(( V^T Y)_{1:P-1},  Z^*\right) &= \max_{\substack{\\ i:( A_P)_{ip} = 0 \\}}  -\sum_{j=1}^{P-1} ( A_P)_{ij}( V^T Y)_j + ( b_P( Z^*))_i.
\end{align*}
\end{proof}
\end{section}

\clearpage

\begin{funding}
This work was supported by the NIH grant 5T32CA009337-40 (Shyr), NSF grants DMS1810829 and DMS2113707 (Parmigiani and Patil), DMS2113426 (Sur) and a William F. Milton Fund (Sur). 
\end{funding}

\begin{code}
Code to reproduce results from the simulations and data application can be found at \texttt{https://github.com/wangcathy/multi-study-boosting}.
\end{code}

\clearpage
\bibliographystyle{imsart-nameyear} 
\bibliography{bibliography}  
\end{document}